\documentclass{article}
\usepackage[title]{appendix}
\usepackage[accepted]{icml2021_oppo}
\usepackage{hyperref}       % hyperlinks
\usepackage{mathtools}
\usepackage{booktabs}       % professional-quality tables
\usepackage{amsfonts}       % blackboard math symbols
\usepackage{nicefrac}       % compact symbols for 1/2, 
\usepackage{thm-restate}
\usepackage{microtype}      % microtypography
\usepackage{xcolor}
\usepackage{graphicx}
\usepackage{enumitem}
\usepackage{amsmath,amsfonts,amsthm,bm}
\usepackage{subcaption,graphicx}
\usepackage{caption}
\usepackage{cleveref}
\usepackage{wrapfig}
\usepackage{stmaryrd}
\usepackage{bm}
\usepackage{amssymb,amsmath,amsthm}

\usepackage{thmtools}
\usepackage{thm-restate}

\theoremstyle{definition}

\theoremstyle{remark}

\renewcommand{\[}{\begin{eqnarray}}
\renewcommand{\]}{\end{eqnarray}}

\usepackage[ruled,vlined,linesnumbered]{algorithm2e}

\usepackage{etoolbox}


\SetKwInput{KwInput}{Input}                % Set the Input
\SetKwInput{KwOutput}{Output}
\usepackage{standalone}
\usepackage{tikz}
\usetikzlibrary{positioning}
\usetikzlibrary{calc,arrows,positioning, fit}

\usepackage{float}
\renewcommand{\[}{\begin{eqnarray}}
\renewcommand{\]}{\end{eqnarray}}

\usepackage{mathbbol}

\usepackage{thm-restate}

\newcommand{\R}{\mathbb{R}}

\newcommand{\F}{\mathcal{F}}
\newcommand{\E}{\mathbb{E}}

\icmltitlerunning{Implicit Acceleration and Feature Learning in Infinitely Wide Neural Networks with Bottlenecks}
\begin{document}

% \maketitle
\twocolumn[
\icmltitle{Implicit Acceleration and Feature Learning in\\Infinitely Wide Neural Networks with Bottlenecks}

\begin{icmlauthorlist}
\icmlauthor{Etai Littwin}{apple}
\icmlauthor{Omid Saremi}{apple}
\icmlauthor{Shuangfei Zhai}{apple}
\icmlauthor{Vimal Thilak}{apple}
\icmlauthor{Hanlin Goh}{apple}
\icmlauthor{Joshua M. Susskind}{apple}
\icmlauthor{Greg Yang}{MSR}
\end{icmlauthorlist}

\icmlaffiliation{apple}{Apple}
\icmlaffiliation{MSR}{MSR}

\icmlcorrespondingauthor{Etai Littwin. }{elittwin@apple.com}

% You may provide any keywords that you
% find helpful for describing your paper; these are used to populate
% the "keywords" metadata in the PDF but will not be shown in the document
\icmlkeywords{neural tangent kernel, NTK, kernel regime, feature learning, infinite neural networks, bottleneck networks}
\vskip 0.3in
]

% this must go after the closing bracket ] following \twocolumn[ ...

% This command actually creates the footnote in the first column
% listing the affiliations and the copyright notice.
% The command takes one argument, which is text to display at the start of the footnote.
% The \icmlEqualContribution command is standard text for equal contribution.
% Remove it (just {}) if you do not need this facility.

\printAffiliationsAndNotice{}  % leave blank if no need to mention equal contribution
%\printAffiliationsAndNotice{\icmlEqualContribution} % otherwise use the standard text.
\begin{abstract}
% In the lazy regime, training dynamics of increasingly wide neural networks under gradient descent converge to kernel gradient descent in function space. In the infinite width limit, the learned model post training becomes independent of the specific instantiation of the parameters, and depends only on the initial function computed by the initialized network (i.e the outputs of the network at initialization). In this limit, training can be analyzed in functional space using the Neural Tangent Kernel. Unfortunately, this limit does not allow actual data dependent feature learning. As we show, when parameterized accordingly (aka using the NTK parametrization), training of 
We analyze the learning dynamics of infinitely wide neural networks with a finite sized bottleneck. Unlike the neural tangent kernel limit, a bottleneck in an otherwise infinite width network allows data dependent feature learning in its bottleneck representation. We empirically show that a single bottleneck in infinite networks dramatically accelerates training when compared to purely infinite networks, with an improved overall performance. We discuss the acceleration phenomena by drawing similarities to infinitely wide deep linear models, where the acceleration effect of a bottleneck can be understood theoretically. 
\end{abstract}
\section{Introduction}
The study of infinitely wide neural networks is one of the most actively researched topics in deep learning theory \cite{NTK,gp1,gp2,gp3,gp4,gp5,gp6,gp7,yang,yang2,yang3,TP2b,Littwin2020OnRK,Littwin2020OnTO}. Previous work identified distinct training regimes that are determined by the networks hyperparameters. In \cite{Yang2020FeatureLI}, it is shown that by correctly scaling hyperparameters such as learning rate, weight multipliers and initialization constants, neural network training dynamics under gradient descent exhibit a limiting behaviour as the width of the network tends to infinity. Two distinct training regimes are identified in the limit. 1) The kernel regime, where the network evolves like a linear model during training. In this regime, intermediate activations in infinite layers change by a vanishing amount, stripping the network of its ability to learn features. 
2) The feature learning regime, where the intermediate activations change in a nontrivial way during training, resulting in data dependent feature learning. \\
In \cite{Yang2020FeatureLI}, a clear dichotomy exists between these regimes in the infinite width limit, where a network is either in one regime or another, but not in both at the same time. A network undergoing training by gradient descent algorithms can be placed in one regime or the other depending on its parametrization at initialization. The underlying assumption in both cases however is that all hidden layers are infinitely wide. 
Some practical architectures incorporate narrow layers by design, or use varying width layers of which some are wide, while others are narrow. For example, bottleneck layers force a network to learn a low dimensional latent representation, and are typically used as part of a generally wide network. Perhaps the most prominent architecture which uses such layers is the autoencoder, where the encoder learns a typically low dimensional representation of its input, while the decoder reconstructs the input from its latent representation.
% A less intuitive example of a bottleneck resides in self attention layers, where the dimensionality of the attention map is preset and determined by the number of tokens in the input, hence can also be viewed as a bottleneck when the number of tokens is considerably smaller than the feature dimension. 
In these types of models, the discrepancy between wide and narrow layers is a carefully hand designed feature of the architecture. Therefore, the standard infinite width approximation cannot be applied without "assuming away" a prominent architectural feature.\\
In this work, we consider a different type of limit, where a bottleneck layer of finite width is inserted in an otherwise infinitely wide network. From a Bayesian perspective, such models have been investigated in \cite{Agrawal2020WideNN,Aitchison2020WhyBI}, however their learning dynamics under gradient descent have not yet been explored to the best of our knowledge. As we show, gradient descent on such models can be understood and simulated exactly in function space rather than parameter space, at a relatively cheap computational cost. This is in contrast to the feature learning limit in \cite{Yang2020FeatureLI}, where it is practically infeasible to simulate exact learning dynamics in the limit without approximations. Empirically, we show that adding a bottleneck layer can significantly boost speed of training convergence. We further speculate on possible reasons for this dramatic boost in convergence speed by drawing equivalence to infinitely wide deep linear models with bottlenecks, where acceleration of the bottlenecks is fully tractable. \\
% \begin{enumerate}
%     \item Theoretical analysis of wide nets with bottlenecks
%     \item Exact training in functional space for toy models
%     \item Exact training in functional space for CNNs and Transformers
%     \item Models trained in parameter space showing additional empirical support
% \end{enumerate}

% \section{Previous Work}
% \begin{itemize}
%     \item NTK theory work including limitations
%     \item Extensions of NTK to self-attention
%     \item Finite width corrections to NTK and Neural Tangent Hierarchy
% \end{itemize}

\section{Neural Networks With Bottlenecks}\label{sec:bottle}

A natural way of thinking about wide neural networks with bottlenecks is through composition of networks. Let $f_n(x;w):\R^{d} \to \R^{d_r}$ denote the (vector) output of a neural network given input vector $x \in \R^d$ and parameters $w$, where $n$ denotes the width of the hidden layers\footnote{We use the term "width" losely here. For MLPs, this means the size of the hidden layers, or the feature dimension in convolutional neural nets.}. 
% The study of infinite width neural networks revolves around the analysis of trained and untrained neural networks in the regime where the width $n$ of the hidden layers of $f_n$ tend to infinity $n \to \infty$. In \cite{NTK}, it is shown that for a wide range of architectures that are properly parametrized (i.e using the NTK parametrization), the outputs of the hidden layers change by an order of $\mathcal{O}(\frac{1}{\sqrt{n}})$ when trained by gradient descent. Roughly speaking, in the limit of $n \to \infty$ the coordinates of the hidden layers do not change, and the entire training dynamics can be fully described in functional space using the renowned NTK equation. 
In the present work, we consider the case where the input $x$ is itself given by the output of a neural network $x(\xi) = g_n(\xi;\theta): \R^{d_0} \to \R^d$, given input vector $\xi \in \R^{d_0}$ and parameters $\theta$, with hidden layers width parametrized by $m$. For simplicity, we assume the width of all hidden layers in both $f$ and $g$ is $n$. The output of the composed architecture is then given by $\F_n(\xi) = (f_n \circ g_n)(\xi)$. Note that even if $n$ is large, the output of $g_n(\xi)$ is still $d$ dimensional, hence $\F_n$ can be viewed as a wide network with a bottleneck of dimension $d$. \\
\paragraph{Setup}
\label{par:setup}
To make things concrete, we consider the case where $f_n$ implement an MLP of depths $2$ and width $n$. Due to technical considerations,\footnote{Rigorously generalizing the claims in this paper to deeper MLPs is not straightforward, and involves a considerable technical challenge. This difficulty arises due to the structure of the input-output Jacobian, which cannot be implemented as a Tensor Program \cite{yang,yang2,yang3} in its current form.} we restrict our discussion in the current paper to a 1 hidden layer MLP for $f$, and leave the rigorous analysis of deeper networks to future work. Given an input vector $g \in \R^d$ and parameters $w = \{u,v\}$, the output $f_n(g)$ is given by:
\begin{align}\label{eqn:mlp}
    f_n(x) = \frac{v}{\sqrt{n}}\phi(\frac{ux}{\sqrt{d}})
\end{align}
where $u \in \R^{n \times d},v \in \R^{d_r \times n}$ are the weight matrices sampled from a normal distribution, and
% For $f$, the parameter matrices are given by $\{w^l\}_{l=1}^{L+f+1}: w^1 \in \R^{n \times d},w^{L_f+1} \in \R^{d_r \times n}, \forall_{1<l\leq L_f}w^l \in \R^{n \times n}$. The output of $f$ is given by:
% \begin{align}\label{eqn:mlp}
% f(g) &= \frac{1}{\sqrt{n}}w^{L_f+1\top} z^{L_f}, z^l = \begin{cases}
% \phi(x^l) &  L\geq l>0\\
% g & l=0
% \end{cases}\\
% x^l &= \begin{cases} \frac{1}{\sqrt{n}}w^{l\top} z^{l-1} & l>1\\
% \frac{1}{\sqrt{d_0}}w^{l\top} z^{l-1} & l=1
% \end{cases}
% \end{align}
$\phi$ is a coordinate nonlinearity which we assume, for the sake of analysis, is twice differentiable with bounded derivatives. We let $g$ implement an arbitrary neural network function with parameter matrices $\theta = \{\theta^l\}_{l=1}^{L}$, with suitable dimensions. 
We are interested in theoretically understanding the behaviour of the composition $\F_n = f_n \circ g_n$ during training in the overparametrized scenario where the width $n$ of the hidden layers tend to infinity, while the bottleneck dimension $d$ remains fixed in size. 
% For the remainder of the paper, a statement involving an infinite width limit should be interpreted by the reader as the limit of $n\to \infty$ and a constant $d$, unless explicitly stated otherwise. \\
Our setup involves the training of the composition function $\F_n$ on a training dataset $\{\xi_i\}_{i=1}^N$, using a loss function $\mathcal{L}$ implicitly containing the labels. We state our results assuming infinitesimal learning rate (aka gradient flow), however we expect our results to carry over to SGD. Since $\F_n$ contains a bottleneck of finite size, we will have to reason about the evolution of forward and backward signals as they propagate through a finite sized layer. 
\paragraph{Notations}
We use $x,\tilde{x}$ to denote generic placeholder vectors to $f$, and $g(\xi)$ as a specific instantiation of $x$ by the output $g(\xi,\theta)$. We denote the input-output Jacobian $J(x) = \frac{\partial f(x)}{\partial x} \in \R^{d_r \times d}$, with the notation $\tilde{J} = J(\tilde{x})$. We use $\F,g,J$ for $\F(\xi), g(\xi), J = J\big(g(\xi)\big)$ where $\xi$ is some generic input. 
To remove clutter, we use subscripts $n,t,i$ to denote the value of a vector/matrix parametrized by $n$ at time $t$ given sample $\xi_i$.  (i.e $g_{n,t,i} = g_{n,t}(\xi_i),g_{n,t} = g_{n,t}(\xi)$ and same for $J$). We use superscripts to denote coordinates of a vector/matrix (i.e $g^\alpha$ is the $\alpha$ coordinate of vector $g$). The absence of a subscript $n$ implies we have taken $n \to \infty$. Finally, we use $\chi_i$ to denote the loss derivative  for sample $i$ (i.e $\chi_i = \frac{\partial \mathcal{L}}{\partial \F_i} \in \R^{d_r}$). 
% We write $f,g$ to denote the functions $f(g):\R^{d}\to \R,g(\xi):\R^{d_0}\to \R^d$ acting separately on their corresponding inputs. We denote a training set and test set by concatenated matrices $D \in \R^{N \times d_0}$ and $D' \in \R^{N' \times d_0}$ where $N,N'$ are the number of samples in the train (resp test) sets.

% That is, for any discussion relating to the function $f$ we assume its input is fixed and deterministic. We consider training the composition function $\F$ using gradient flow on a training dataset $\{\xi_i\}_{i=1}^N$ using a loss function $\mathcal{L}()$

% We use subscript to denote a dependency on sample (i.e $g_i = g(\xi_i)$). With a slight abuse of notation, we use subscript $t$ to denote a value of a tensor at time $t$ of training. We use a superscript to denote indexing

\subsection{Preliminaries}\label{sec:pre}
As width increases, pre-activations of initialized neural networks without bottlenecks approach a centered gaussian process (GP), independent across coordinates but possibly correlated across inputs. At the output level, the following hold at initialization:
\begin{align}\label{eqn:nngp}
\lim_{n \to \infty} g_n(\xi) \overset{d}{=} g(\xi),&~~~ \lim_{n \to \infty} f_n(x) \overset{d}{=} f(x)
\end{align}
where $g(\xi),f(x)$ are GPs with respect to their respective inputs. For $f$ as implemented in \cref{eqn:mlp}, we can write the defining covariance $\Lambda$ of the GP given input samples $x,\tilde{x}$:
\begin{align}
\forall_\alpha,\begin{pmatrix}
f^\alpha(x)\\
f^\alpha(\tilde{x}) 
\end{pmatrix} &\sim \mathcal{N}\big(0, \Lambda(x,\tilde{x})\big)\\
\Lambda(x,\tilde{x}) &= \begin{pmatrix} \Sigma(x,x) & \Sigma(x,\tilde{x})\\ \Sigma(\tilde{x},x) & \Sigma(\tilde{x},\tilde{x}) \end{pmatrix} \in \R^{2 \times 2}
% \Sigma(x,\tilde{x}) &= \E_{a,b \sim \mathcal{N}(\Lambda_0(x,\tilde{x}))}\big[\phi(a)\phi(b)\big]\\
% \Lambda_0(x,\tilde{x}) &= \begin{pmatrix} \frac{\|x\|^2}{d} & \frac{x^\top \tilde{x}}{d}\\ \frac{\tilde{x}^\top x}{d} & \frac{\|\tilde{x}\|^2}{d} \end{pmatrix} \in \R^{2 \times 2}
\end{align}

We assume both $f,g$ have empirical NTKs $\mathcal{K}_n ,\Theta_n$ defined as $\mathcal{K}_n(x,\tilde{x}) = \frac{\partial f(x)}{\partial w}\frac{\partial  f(\tilde{x})^\top}{\partial w} \in \R^{d_r \times d_r},\Theta_n(\xi,\tilde{\xi}) = \frac{\partial g(\xi)}{\partial \theta}\frac{\partial g(\tilde{\xi})^\top}{\partial \theta} \in \R^{d \times d}$
with corresponding limits $\mathcal{K}(x,\tilde{x})I_{d_r}, \Theta(\xi,\tilde{\xi})I_{d}$, where $I_{d_r},I_d$ are identity matrices of size $d_r,d$, and $\mathcal{K}(x,\tilde{x}),\Theta(\xi,\tilde{\xi}) \in \R$.
% Intuitively, the composition $\F(\xi) = (f \circ g)(\xi)$ should not be expected to exhibit a GP behaviour. Since $g$ is finite dimensional, the following layer $ \frac{ug}{\sqrt{d}}$ is given by a linear combination of the same finite stochastic vector $g$, hence the coordinates of $ug$ become statistically dependent even in the limit of $n \to \infty$. 
Consistent with intuition, it was rigorously shown in \cite{Agrawal2020WideNN} that as the width increases the output of an MLP with bottlenecks converges to a composition of GPs. \\
\paragraph{Training} We consider training the model described in \cref{sec:bottle}. Under gradient flow, the weights evolve according to $\dot{w}_t = -\nabla_{w_t}\mathcal{L}_t,~~~\dot{\theta}_t = -\nabla_{\theta_t}\mathcal{L}_t$. The evolution of the output of the composition function can be described by the following set of ODEs:  
\begin{align}\label{eqn:gf}
\dot{\mathbf{\F}}_{n,t}(\xi) &= \frac{\partial \F_{n,t}(\xi)}{\partial w_t}\dot{w}_t + \frac{\partial \F_{n,t}(\xi)}{\partial \theta_t}\dot{\theta}_t
\end{align}
Substituting the empirical kernel definitions:
\begin{align}\label{eqn:random_kernel}
\dot{\F}_{n,t}(\xi) &= -\sum_{i=1}^N\mathcal{K}_{n,t}\big(g_{n,t},g_{n,t,i}\big)\chi_{n,t,i} \\\nonumber
&-\sum_{i=1}^NJ_{n,t} \Theta_{n,t}(\xi,\xi_i)J^{\top}_{n,t,i} \chi_{n,t,i}
\end{align}
where $g_{n,t} = g_{n,t}(\xi), J_{n,t} = J_{n,t}(\xi)$. The above evolutionary equations can be interpreted as Kernel equations with random, evolving kernels that depend on the weights $w,\theta$. In \cite{NTK}, it was shown for infinite width networks (without bottlenecks) the output is fully deterministic at any time $t$ conditioned on the initial GP output at time $t=0$. In contrast, since both bottleneck embedding $g$ and input-output Jacobian $J$ are finite, we expect \cref{eqn:random_kernel} to remain random even when taking the limit $n \to \infty$. To get a more complete view of the evolution of the composition function $\F$ during training at the limit, we must reason about the dynamics of the Jacobian term $J$.

\section{Dynamics in Function Space} 
A key observation in our analysis is that the input-output Jacobian $J(x)$ converges to a multivariate GP, and evolves as a linear function in the infinite width limit, similarly to outputs of an infinitely wide network. Here, $J(x)$ will have non trivial correlations across its coordinates, unlike layers in the NTK limit where the coordinates are independent. Our first result relating to the initial state of $J$ is stated in the following proposition:

\begin{restatable}[GP behaviour of the Jacobian]{lemma}{GP}\label{lemma:GP}
For $f_n(x)$ and its limit $f(x)$ as described in \cref{eqn:mlp,eqn:nngp}, the following holds at initialization for every pair of fixed inputs $x,\tilde{x}$:
\[
\lim_{n \to \infty}J_n(x) = J(x) \label{eqn:jac}
\]
where $J(x)$ is a multivariate GP with independent rows, and $J(x),f(x)$ are jointly Gaussian 
with:
\begin{align}
\begin{pmatrix}\label{JJ:cov}
J^{\alpha,\beta}(x)\\
J^{\alpha,\gamma}(\tilde{x})
\end{pmatrix} &\sim \mathcal{N}(\bm{0}, \begin{pmatrix} \Sigma_{(2)}^{\beta,\beta}(x,x) & \Sigma_{(2)}^{\beta,\gamma}(x,\tilde{x})\\ \Sigma_{(2)}^{\gamma,\beta}(\tilde{x},x) & \Sigma_{(2)}^{\gamma,\gamma}(\tilde{x},\tilde{x}) \end{pmatrix}\\
\begin{pmatrix}\label{fJ:cov}
f^\alpha(x)\\
J^{\alpha,\beta}(\tilde{x})
\end{pmatrix} &\sim \mathcal{N}(\bm{0}, \begin{pmatrix} \Sigma(x,\tilde{x}) & \Sigma_{(1)}^{\beta}(x,\tilde{x})\\ \Sigma_{(1)}^\beta(x,\tilde{x})^\top & \Sigma_{(2)}^{\beta,\beta}(\tilde{x},\tilde{x}) \end{pmatrix}
\end{align}
where:
\begin{align}\label{sig1:sig2:def}
\Sigma_{(1)}(x,\tilde{x}) &= \frac{\partial}{\partial b}\Sigma(x,b)\Big|_{b=\tilde{x}} \in \R^{1\times d}\\
\Sigma_{(2)}(a,b) &= \frac{\partial^2}{\partial a\partial b}\Sigma(a,b)\Big|_{a=x,b=\tilde{x}} \in \R^{d \times d}
\end{align}
\end{restatable}

\cref{lemma:GP} already illustrates a novel aspect of infinite width networks. As the input-output Jacobian is frequently used to derive sensitivity to perturbations, \cref{lemma:GP} shows that sensitivity to perturbations and outputs can be jointly modeled as a multivariate GP for sufficiently wide models.  \\
In the next theorem, we show that the $J(x)$ evolves linearly in wide models, in a similar fashion to the network outputs. Now we are ready to characterize the full training dynamics of $\F$ in the infinite width limit in the following theorem:
\begin{restatable}[Evolution of composed function]{thm}{main}\label{thm:main}
For networks $f,g$ as in \cref{sec:bottle}, the following ODEs describe the dynamics of $\F,g,J$ in the limit of $n \to \infty$:
\begin{align}
\dot{g}_t(\xi) &= -\sum_{i=1}^N\Theta(\xi,\xi_i)J_{t,i}^\top\chi_{t,i} \label{eqn:eq1}\\
\dot{J}_t(x) &= - \sum_{i=1}^N\chi_{t,i}\Xi(x,g_{t,i})^\top \label{eqn:eq2}\\
\dot{\F}_t(\xi) &= -\sum_{i=1}^N\Big[\mathcal{K}\big(g_t,g_{t,i}\big)I_{d_r} + 
\Theta(\xi,\xi_i)J_tJ_{t,i}^{\top}\Big]\chi_{t,i} \label{eqn:eq3}
\end{align}
where $\Xi(-,-)$ is a deterministic function defined as $\Xi(x,\tilde{x}) \in \R^{ d } = \lim_{n \to \infty} \frac{\partial^\top}{\partial x}\mathcal{K}_n^{\alpha,\alpha}(x,\tilde{x})$.
\end{restatable}
For $\phi = \text{ReLU}$, we give an explicit form for $\Xi$ in \cref{sec:exp}. 
ODEs in \cref{eqn:eq1,eqn:eq2,eqn:eq3} depend on deterministic, frozen kernels $\mathcal{K},\Theta,\Xi$. The evolution of $\F,g,J$ is hence completely deterministic after conditioning on initial states, and can therefore be expressed in functional space.
% The resulting ODEs however are non analytic in the general case, hence we must resort to numerical simulations.

% \paragraph{Stochastic Gradient Descent}
% To perform empirical experiments, we resort to the discrete time version of \cref{eqn:eq1,eqn:eq2,eqn:eq3} which will allow simulating SGD on the loss in function space. Let $\xi^t$ denote the sample fed into the SGD algorithm at step $t$. The discrete evolution of $\F,g,J$ is given by:
% \begin{align}
% \Delta g_{t+1}(\xi) &= -\sum_{i=1}^t\Theta(\xi,\xi_i)J_{t,i}^\top\chi_{t,i} \label{eqn:d1}\\
% \Delta J_{t+1}(x) &= - \sum_{i=1}^t\chi_{t,i}\Xi(x,g_{t,i})^\top \label{eqn:d2}\\
% \Delta \F_{t+1}(\xi) &= -\sum_{i=1}^t\Big[\mathcal{K}\big(g_t,g_{t,i}\big)I_{d_r} + 
% \Theta(\xi,\xi_i)J_tJ_{t,i}^{\top}\Big]\chi_{t,i} \label{eqn:d3}
% \end{align}
% where $\Delta \bullet_{t+1} = \bullet_{t+1} - \bullet_t$. 

\section{Implicit Acceleration by Bottlenecks in Linear Networks}\label{sec:implicit}
To intuitively understand the training aspects of introducing bottlenecks in infinite width networks, we draw inspiration from deep linear networks. Assume $f_n,g_n$ implement deep linear MLPs of depth $L_f,L_g$ respectively and width $n$, and let $w_{\text{eff}} \in \R^{d_r \times d} = \frac{1}{\sqrt{n^{(L_f-1)} d}}w^{L_f}w^{L_{f-1}}...w^1,~\theta_\text{eff} \in \R^{d \times d_0} = \frac{1}{\sqrt{n^{(L_g-1)} d_0}}\theta^{L_g}\theta^{L_{g-1}}...\theta^1$. Hence, we have that $g(\xi) = \theta_\text{eff} \xi$, $f(g) = w_{\text{eff}} g$ and $\F(\xi) = w_{\text{eff}}\theta_\text{eff} \xi$. In finite networks, recent results have shown that in some cases, the stacking of linear layers produces an acceleration effect, and a low rank bias when optimized by gradient descent \cite{Arora2018OnTO}. Moreover, the acceleration effect is akin to momentum, and cannot be reproduced by adding some regularizer to the objective.  However, the acceleration effect as outlined in \cite{Arora2018OnTO} disappears when considering the NTK regime. This is because in this regime, training speed is determined by the NTK itself, which does not change meaningfully with stacking of additional linear layers. Indeed, for a linear $g$ we have that $\Theta(\xi,\tilde{\xi}) = L_g\frac{\xi^\top\tilde{\xi}}{d_0}$. However, by introducing a bottleneck, we regain the lost acceleration effect, as illustrated in the following lemma.
\begin{restatable}{lemma}{linear}\label{lemma:linear}
In the limit of $n \to \infty$, optimizing $\F$ by running gradient flow on the weights $\{w^l\}_{l=1}^{L_f}$ with a learning rate $\epsilon_f$, and  $\{\theta^l\}_{l=1}^{L_g}$ with a learning rate $\epsilon_g$, is equivalent to running gradient flow directly on the effective weights $w_{\text{eff}}, \theta_\text{eff}$ with learning rates $L_f\epsilon_f,L_g\epsilon_g$.
\end{restatable}
 \cref{lemma:linear} suggests an infinitely wide deep linear network with a bottleneck is essentially reduced to a two layer, finite linear network with weight matrices $w_{\text{eff}}, \theta_\text{eff}$ under gradient descent. Therefore, under mild initialization conditions,\footnote{The acceleration effect formally requires a small initialization and learning rates to hold. Empirically, these conditions may sometimes be relaxed.} a bottleneck brings about accelerated learning in linear networks. It is worth noting, additional capacity cannot be attained by stacking additional layers in linear networks. Moreover, shallow and deep linear functions represent the same function class. Hence training speed can be directly attributed to trajectories of gradient descent. This does not hold for nonlinear finite networks, where changes in depth or width also affect capacity. However, in the infinite width regime, with bottlenecks or not, capacity is infinite. Hence, we can isolate the effect of bottlenecks on training and acceleration without confounding capacity.
 
%  However, we note that in the kernel regime a network has "infinite" capacity to fit the training data \footnote{Provided the NTK is positive definite}, which can only be reduced by the addition of a bottleneck. We therefore hypothesise that the implicit acceleration observed in both linear and nonlinear (wide) networks with bottlenecks is the cause of the same underlying implicit bias of depth.

\section{Experiments}

We provide empirical support for our theoretical contributions in three parts:
\begin{itemize}
    \item In part 1, we conduct simulations to numerically verify the theory in \cref{lemma:GP,lemma:linear,thm:main}. We present these results in \cref{verify}.
    \item  In part 2, we train infinite neural networks with bottlenecks on MNIST \cite{mnist} and CIFAR-10 \cite{cifar} datasets by simulating SGD on the loss in function space, investigating training acceleration effects and test performance. \Cref{fig:mnist_cifar_train_loss} summarizes the results from these experiments. We observe that the accelerated training predicted by  \cref{lemma:linear} for linear models is also visible when we train infinite width nonlinear networks with bottlenecks on the two real world datasets.\footnote{Figure~\ref{fig:cifar10_lr1k_metrics:train_loss_main} show loss for first 15K steps on CIFAR-10, as some training runs did not complete in time due to compute resource scarcity. Extended results showing loss trajectories for models trained longer are available in figure~\ref{fig:cifar10_lr1k_metrics} in \cref{sec:exp}.} We present additional results and analysis in \cref{sec:exp}.
    %\item In part 3, we run experiments with finite width networks trained with standard SGD to verify whether the acceleration effect holds in this setting. These results are presented in \cref{finiteapprox}.
    \item In part 3, we run experiments with finite width networks trained with standard SGD and verify that the acceleration effect holds in this setting. These results are presented in \cref{finiteapprox}.
\end{itemize}

% \subsection{Training infinite width bottlenecks in function space}
\label{expt:data:mnist}
\begin{figure}[!t]
    %\centering % Not needed
    \begin{subfigure}[b]{\columnwidth}
        \includegraphics[width=\columnwidth]{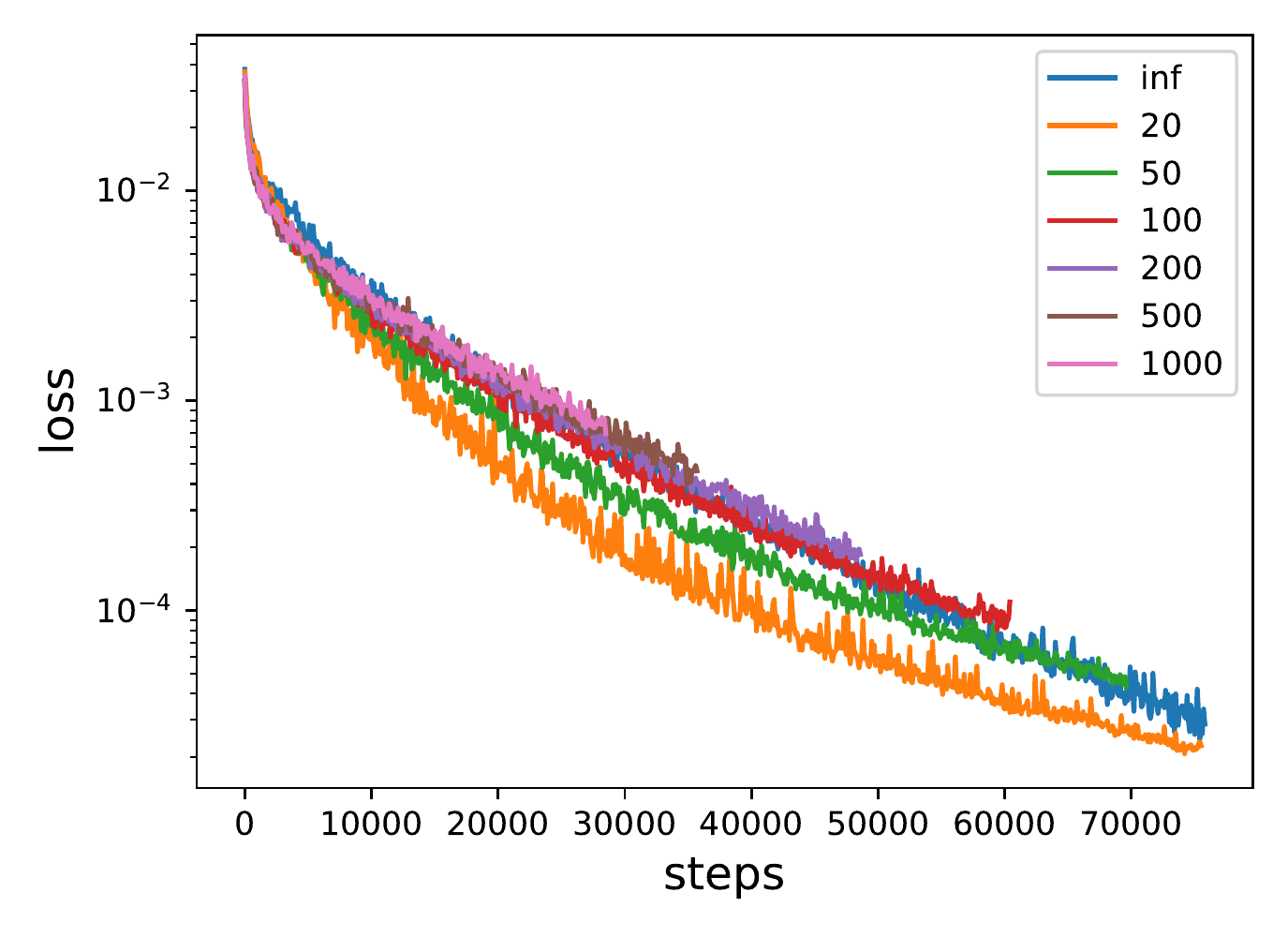}
        \caption{MNIST training loss}
        % with learning rate 250
        \label{fig:mnist_lr5k_metrics:train_loss_main}
    \end{subfigure}
    \begin{subfigure}[b]{\columnwidth}
        \includegraphics[width=\columnwidth]{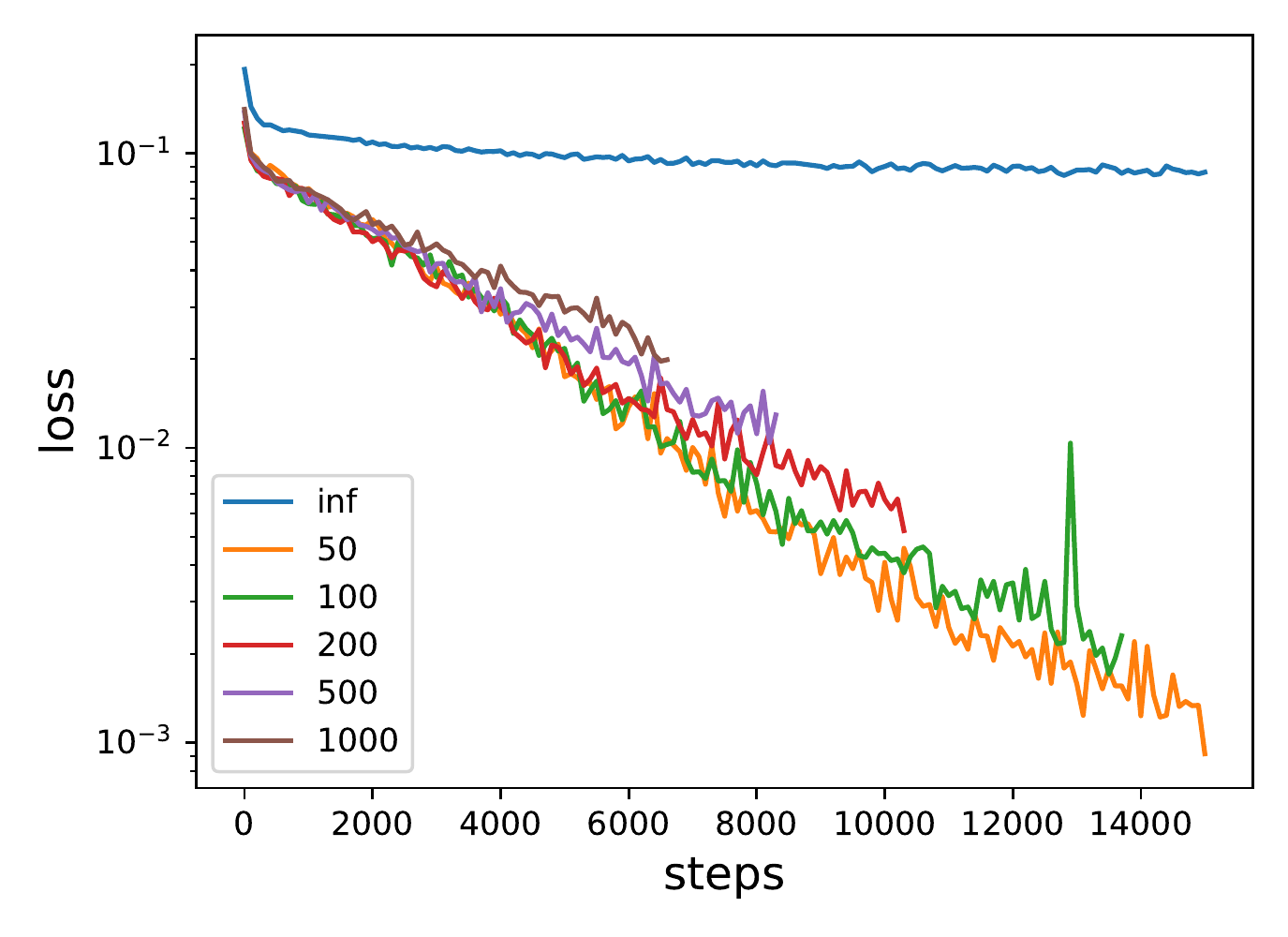}
        \caption{CIFAR-10 training}
        % with learning rate 1000
        \label{fig:cifar10_lr1k_metrics:train_loss_main}
    \end{subfigure}

    \caption{{Training loss for infinite width bottleneck models with widths from smallest to largest (inf indicates infinite width bottleneck). Loss is reduced in function space by simulating SGD with \cref{eqn:eq1,eqn:eq2,eqn:eq3}}}%
    \label{fig:mnist_cifar_train_loss}
\end{figure}

% \paragraph{Loss and accuracy evolution} Figure \ref{fig:mnist_cifar_train_loss} shows the evolution of training and test loss and accuracy over the course of training for architectures where the bottleneck width is varied from ``no bottleneck'' (labeled as inf) to 1000. The models are trained with a learning rate set to 250. We observe that training loss decays rapidly for networks with narrower bottlenecks. Figure \ref{fig:mnist_lr5k_heatmap} (\cref{sec:exp}) further examines training and test loss evolution for the first 4000 steps, confirming that training accelerates as bottleneck size narrows. Note that while narrower bottleneck models exhibit lower test loss compared to their wider bottleneck counterparts, this does not hold for test accuracy, potentially due to a loss-metric mismatch as is commonly observed in practice \cite{pmlr-v97-huang19f}.
% These results are inline with the analysis of linear networks with bottlenecks, that wide networks with bottlenecks accelerate training, and confirm empirically that this effect holds in practical nonlinear settings.

% Additional results are available in \cref{sec:exp} that include:
% \begin{itemize}
%     \item regression with syntehtic data in \cref{expt:data:synthetic}
%     \item binary classification with CIFAR-10 data in \cref{expt:data:cifar10}
% \end{itemize}

\section{Conclusion}
In this work we investigate the effect of applying a bottleneck in an otherwise infinite width network. We do this by first deriving the ODEs corresponding to optimization of such a model under gradient flow. Our theoretical analysis reveals novel insights regarding the behaviour of input-output Jacobians, both at initialization and training. Though stated for shallow, single hidden layer networks post bottleneck, we expect our results to hold in more general cases. Empirically, we observe that infinite width networks with bottlenecks train much faster than their fully infinite counterparts, while typically achieving better overall performance. We hope our results pave the way for new understanding of learning dynamics beyond kernel regimes.

\bibliography{reference}
\bibliographystyle{icml2021}

\newpage
\appendix
\onecolumn
\section{Proof of \cref{lemma:GP}}
\GP*
\begin{proof}
Intuitively, the derivative process of any centered GP with a covariance kernel function $\mathcal{K}(x,\tilde{x})$ is another GP, provided that the kernel function $\mathcal{K}$ is everywhere differentiable. Since the Jacobian is by definition the derivative of the output with respect to the input, this implies that \cref{eqn:jac} holds, along with the stated covariances. However, the subtlety here is to show that the empirical Jacobian indeed converges in distribution to the derivative of the NN-GP in the limit. For a single hidden layer MLP, the convergence of the Jacobian to a GP can be established by a straightforward application of the standard CLT theorem, as in the feed forward case. \\
The Jacobian takes the form:
\[
J_n(x) = \frac{\partial f(x)}{\partial x} = \frac{v}{\sqrt{nd}} \text{diag}\big(\phi'(\frac{ux}{\sqrt{d}})\big)u
\]
Each element $J_n^{\alpha,\beta}(x)$ is hence given by:
\[
J_n^{\alpha,\beta}(x) = \frac{1}{\sqrt{nd}}\sum_{\gamma = 1}^n v^{\alpha,\gamma}\phi'^{\gamma}(\frac{ux}{\sqrt{d}})u^{\gamma,\beta}
\]
Note that the sequence $\{v^{\alpha,\gamma}\phi'^{\gamma}(ux)u^{\gamma,\beta}\}_{\gamma=1}^n$ is a sequence of zero mean iid variables. From the polynomial boundness of $\phi$, the coordinates of $\phi'$ have bounded second moment. We can then apply the CLT theorem to establish a convergence to a GP with some kernel function $\Sigma^\star \in \R^{3 \times 3}$:
\[
\begin{pmatrix}
f^\alpha(x)\\
J^{\alpha,\beta}(\tilde{x})\\
J^{\alpha,\delta}(x')
\end{pmatrix}
\overset{p}{=} \lim_{n \to \infty} \frac{1}{\sqrt{n}}\begin{pmatrix}
\sum_{\gamma=1}^n v^{\alpha,\gamma} \phi^{\gamma}(\frac{ux}{\sqrt{d}})\\
\frac{1}{\sqrt{d}}\sum_{\gamma = 1}^n v^{\alpha,\gamma}\phi'^{\gamma}(\frac{u\tilde{x}}{\sqrt{d}})u^{\gamma,\beta}\\
\frac{1}{\sqrt{d}}\sum_{\gamma = 1}^n v^{\alpha,\gamma}\phi'^{\gamma}(\frac{ux'}{\sqrt{d}})u^{\gamma,\delta}
\end{pmatrix} \overset{p}{=} \mathcal{N}(\bm{0},\Sigma^\star)
\]
From the independence of $u,v$, it trivially holds that $\forall_{\alpha,\beta}\E[J^{\alpha,\beta}(x)] = 0$, $\forall_{\alpha \neq \gamma,\beta,\delta}\E[J^{\alpha,\beta}(x)J^{\gamma,\delta}(\tilde{x})] = 0$ and $\forall_{\alpha \neq \gamma,\beta}\E[J^{\alpha,\beta}(x)f(\tilde{x})^\gamma] = 0$. For any two inputs $x,\tilde{x}$ we have:
\begin{align}
\forall_{\alpha,\beta, \delta}\E[J^{\alpha,\beta}(x)J^{\alpha,\delta}(\tilde{x})] &= \frac{1}{d}\E\Big[ \big(v^{\alpha,\gamma}\big)^2\phi'(\frac{ux}{\sqrt{d}})^\gamma \phi'(\frac{u\tilde{x}}{\sqrt{d}})^\gamma u^{\gamma,\beta}u^{\gamma,\delta}\Big]\\
&= \frac{1}{d}\E\Big[ \phi'(\frac{ux}{\sqrt{d}})^\gamma \phi'(\frac{u\tilde{x}}{\sqrt{d}})^\gamma u^{\gamma,\beta}u^{\gamma,\delta}\Big]\\
&= \E\Big[ \frac{\partial}{\partial x^\beta}\frac{\partial}{\partial \tilde{x}^\delta}\phi(\frac{ux}{\sqrt{d}})^\gamma \phi(\frac{u\tilde{x}}{\sqrt{d}})^\gamma \Big]
\end{align}

From our assumption that $\phi$ is differentiable with bounded derivatives, it holds that $\E\Big[ \big|\frac{\partial}{\partial x^\beta}\frac{\partial}{\partial \tilde{x}^\delta}\phi(\frac{ux}{\sqrt{d}})^\gamma \phi(\frac{u\tilde{x}}{\sqrt{d}})^\gamma \big|\Big] < \infty$. 
By the dominated convergence theorem:
\[
\E\Big[ \frac{\partial}{\partial x^\beta}\frac{\partial}{\partial \tilde{x}^\delta}\phi(\frac{ux}{\sqrt{d}})^\gamma \phi(\frac{u\tilde{x}}{\sqrt{d}})^\gamma \Big] = \frac{\partial}{\partial x^\beta}\frac{\partial}{\partial \tilde{x}^\delta}\E\Big[ \phi(\frac{ux}{\sqrt{d}})^\gamma \phi(\frac{u\tilde{x}}{\sqrt{d}})^\gamma \Big] = \frac{\partial^2}{\partial x^\beta \partial \tilde{x}^\delta}\Sigma(x,\tilde{x})
\]
Similarly:
\begin{align}
\forall_{\alpha,\beta, \delta}\E[J^{\alpha,\beta}(x)f^\alpha(\tilde{x})] &= \E\Big[\big(v^{\alpha,\gamma}\big)^2\phi'(\frac{ug}{\sqrt{d}})^\gamma \phi(\frac{u\tilde{x}}{\sqrt{d}})^\gamma u^{\gamma,\beta}\Big]\\
&= \frac{\partial}{\partial x^\beta}\E\Big[\phi(\frac{ux}{\sqrt{d}})^\gamma \phi(\frac{u\tilde{x}}{\sqrt{d}})^\gamma \Big] = \frac{\partial}{\partial x^\beta}\Sigma(x,\tilde{x})
\end{align}
which concludes the proof.
\end{proof}

\section{Proof of \cref{thm:main}}
\main*

\begin{proof}
The empirical dynamical equation for $\F,g$ are given in \cref{eqn:random_kernel}:
\begin{align}
\dot{g}_{n,t}(\xi) &= -\sum_{i=1}^N\Theta_{n,t}(\xi,\xi_i)J_{n,t,i}^\top\chi_{n,t,i}\\
\dot{\F}_{n,t}(\tilde{\xi}) &= -\sum_{i=1}^N\mathcal{K}_{n,t}\big(g_{n,t},g_{n,t,i}\big)\chi_{n,t,i} 
-\sum_{i=1}^NJ_{n,t} \Theta_{n,t}(\tilde{\xi},\xi_i)J^{\top}_{n,t,i} \chi_{n,t,i}
\end{align}

Prior results regarding convergence of NTK functions have been established pointwise on a fixed dataset \cite{yang2}. In our scenario, we have both the bottleneck embeddings and the Jacobian which evolve continuously during training. This would not pose a major problem to us since a pointwise convergence implies local uniform convergence over a closed region. All we need is insure both the bottleneck embedding terms $g$ and Jacobian terms $J$ do explode during training.
% Note that $J_{n,t,i}$ depends on the weights $u,v$ only through the outputs $g$. Hence, conditioned on the bottleneck embeddings $g_{n,0}$, $\Theta_{n,0}$ is independent of the Jacobian $J^{\top}_{n,0}$, and both functions $\mathcal{K}_{n,t},\Theta_{n,t}$ converge almost surely to their respective limits $\mathcal{K},\Theta$.
Assuming this is the case, it is straightforward that:
\begin{align}
    \lim_{n \to \infty } \sum_{i=1}^N\Theta_{n,t}(\xi,\xi_i)J_{n,t,i}^\top\chi_{n,t,i} &= \sum_{i=1}^N\Theta(\xi,\xi_i)J_{t,i}^\top\chi_{t,i}\\
    \lim_{n \to \infty}\sum_{i=1}^NJ_{n,t} \Theta_{n,t}(\tilde{\xi},\xi_i)J^{\top}_{n,t,i}\chi_{n,t,i} &= \sum_{i=1}^N \Theta(\tilde{\xi},\xi_i)J_tJ^{\top}_{t,i} \chi_{t,i} 
\end{align}

% As the convergence of the first term, known results regarding NTK convergence have been established for a fixed dataset, with the kernel functions converge pointwise almost surely \cite{yang2}. In our scenario, we have both the bottleneck embeddings and the Jacobian which evolve continuously during training. This would not pose a major problem to us since a pointwise convergence implies local uniform convergence over a closed region. All we need is insure both the bottleneck embeddings and Jacobians do explode during training. \\
We now have to show how the bottleneck outputs evolve by deriving the dynamical equation for the Jacobian. 
This is done by taking the input derivative of the dynamical equation for $f(x)$. Namely:
\begin{align}
    \dot{J}_{n,t}(x) = \frac{\partial}{\partial x}\dot{f}_{n,t}(x) = -\frac{\partial}{\partial x}\sum_{i=1}^N\mathcal{K}_{n,t}(x,g_{n,t,i})\chi_{n,t,i} 
\end{align}
We will now prove that $x,\tilde{x}$, $\lim_{n \to \infty} \frac{\partial \mathcal{K}_{n,t}(x,\tilde{x})}{\partial x} = \frac{\partial \mathcal{K}(x,\tilde{x})I_{d_r}}{\partial x}$.

Let $y(x) = \frac{ux}{\sqrt{d}}, z(x) = \phi \big(y(x)\big)$.
For our 1 hidden layer MLP, the kernel $\mathcal{K}_n$ is given by:
\begin{align}
    \mathcal{K}_n(x,\tilde{x}) = \frac{\partial f_n(x)}{\partial w}\frac{\partial  f_n(\tilde{x})^\top}{\partial w} = \frac{x^\top \tilde{x}}{d}\cdot \frac{1}{n}\frac{\partial f(x)}{\partial y_n(x)}\frac{\partial f_n(\tilde{x})^\top}{\partial y_n(\tilde{x})} + \frac{z_n(x)^\top z_n(\tilde{x})}{n}I_{d_r}
\end{align}
Hence, we have that:
\begin{align}
    \frac{\partial}{\partial x^\alpha}\mathcal{K}_n(x,\tilde{x}) = \frac{\tilde{x}^\alpha}{d}\cdot \frac{1}{n}\frac{\partial f_n(x)}{\partial y_n(x)}\frac{\partial f_n(\tilde{x})^\top}{\partial y_n(\tilde{x})} + \frac{x^\top \tilde{x}}{d}\cdot \frac{1}{n}\frac{\partial^2 f_n(x)}{\partial y_n(x) \partial x^\alpha}\frac{\partial f_n(\tilde{x})^\top}{\partial y_n(\tilde{x})} + \frac{1}{n}\frac{\partial z_n(x)^\top}{\partial x^\alpha}z_n(\tilde{x})I_{r_d}
\end{align}

We now handle each term separately. For the first term, we have the trivial limit which holds throughout training:
\begin{align}
    \frac{\tilde{x}^\alpha}{d}\cdot \frac{1}{n}\frac{\partial f_n(x)}{\partial y_n(x)}\frac{\partial f_n(\tilde{x})^\top}{\partial y_n(\tilde{x})} \to \frac{\tilde{x}^\alpha}{d}\dot{\Sigma}(x,\tilde{x})
\end{align}
where:
\begin{align}
    \dot{\Sigma}(x,\tilde{x}) &= \E_{a,b \sim \mathcal{N}(\bm{0},\Lambda_0(x,\tilde{x}))}\big[\dot{\phi}(a)\dot{\phi}(b)\big],~~~
    \Lambda_0(x,\tilde{x}) = \begin{pmatrix} \frac{\|x\|^2}{d} & \frac{x^\top \tilde{x}}{d}\\ \frac{\tilde{x}^\top x}{d} & \frac{\|\tilde{x}\|^2}{d} \end{pmatrix}
\end{align}
For the second term:
\begin{align}
    \frac{\partial f_n^\beta(x)}{\partial y_n^\gamma(g)} = v^{\beta,\gamma} \phi'^\gamma(\frac{ux}{\sqrt{d}}),~~~\frac{\partial^2 f_n^\beta(x)}{\partial y_n^\gamma(x) \partial x^\alpha} = \frac{1}{\sqrt{d}}v^{\beta,\gamma}u^{\gamma,\alpha} \phi''^\gamma(\frac{ux}{\sqrt{d}})
\end{align}
hence at the limit:
\begin{align}
    \frac{1}{n}\frac{\partial^2 f_n^\beta(x)}{\partial y_n(x) \partial x^\alpha}\frac{\partial f_n^\delta(\tilde{x})^\top}{\partial y_n(\tilde{x})} &= \frac{1}{\sqrt{d}}\frac{\sum_{\gamma=1}^n v^{\beta,\gamma}v^{\delta,\gamma}u^{\gamma,\alpha} \phi''^\gamma(\frac{ux}{\sqrt{d}})\phi'^\gamma(\frac{u\tilde{x}}{\sqrt{d}})}{n} \label{eqn:clt}\\
    &\to \frac{1}{\sqrt{d}}\E\Big[u^{\gamma,\alpha}\phi''^\gamma(\frac{ux}{\sqrt{d}})\phi'^\gamma(\frac{u\tilde{x}}{\sqrt{d}}) \Big]\mathbb{1}(\beta = \delta)\label{eqn:before}\\
    &= \E\Big[\frac{\partial}{\partial x^\alpha} \phi'^\gamma(\frac{ux}{\sqrt{d}})\phi'^\gamma(\frac{u\tilde{x}}{\sqrt{d}}) \Big]\mathbb{1}(\beta = \delta)\label{eqn:after}\\
    &= \frac{\partial}{\partial x^\alpha}\E\Big[ \phi'^\gamma(\frac{ux}{\sqrt{d}})\phi'^\gamma(\frac{u\tilde{x}}{\sqrt{d}}) \Big]\mathbb{1}(\beta = \delta)\\
    &= \frac{\partial}{\partial x^\alpha}\dot{\Sigma}(x,\tilde{x})\mathbb{1}(\beta = \delta)
\end{align}
where we used the independence of the coordinates of $\phi',\phi''$ to apply the LLN (law of large numbers) theorem on \cref{eqn:clt}, and used the dominated convergence theorem to get from \cref{eqn:before} to \cref{eqn:after}. We arrive at the result for the second term:
\begin{align}
    \frac{x^\top \tilde{x}}{d}\cdot \frac{1}{n}\frac{\partial^2 f_n(x)}{\partial y_n(x) \partial x^\alpha}\frac{\partial f_n(\tilde{x})^\top}{\partial y_n(\tilde{x})} \to \frac{x^\top \tilde{x}}{d}\frac{\partial}{\partial x^\alpha}\dot{\Sigma}(x,\tilde{x})I_{d_r}
\end{align}

For the final term we have similarly:
\begin{align}
    \frac{1}{n}\frac{\partial z_n(x)^\top}{\partial x^\alpha}z_n(\tilde{x})I_{r_d} = \frac{1}{\sqrt{d}}\frac{\sum_{\gamma =1}^n u^{\gamma,\alpha}\phi'^\gamma(\frac{ux}{\sqrt{d}})\phi^\gamma(\frac{u\tilde{x}}{\sqrt{d}})}{n} \to \frac{\partial}{\partial x^\alpha}\Sigma(x,\tilde{x})
\end{align}
where we again used the independence of the coordinates of $\phi$ and applied LLN to the sum. Wrapping up all three terms, we arrive at:
\begin{align}
    \lim_{n \to \infty}\frac{\partial}{\partial x^\alpha}\mathcal{K}_n(x,\tilde{x}) = \frac{\partial}{\partial x^\alpha}\Big(\frac{x^\top \tilde{x}}{d}\dot{\Sigma}(x,\tilde{x}) + \Sigma(x,\tilde{x})\Big) = \frac{\partial}{\partial x^\alpha}\mathcal{K}(x,\tilde{x}) = \Xi^\alpha(x,\tilde{x})
\end{align}
concluding the proof.
\end{proof}

\section{Proof of \cref{lemma:linear}}
\linear*
\begin{proof}
We prove the claim by deriving the update equations using \cref{thm:main} for both cases and show their equivalence.
Concretely, we must show that the functions $\Theta,\mathcal{K},\Xi$ are the same in both cases ups to constants. When optimizing the effective weights $w_{\text{eff}},\theta_{\text{eff}}$ directly, we trivially have that $\Theta(\xi,\tilde{\xi}) = \frac{\xi^\top \tilde{\xi}}{d_0}, \mathcal{K}(x,\tilde{x}) = \frac{x^\top \tilde{x}}{d}$ and $\Xi(x,\tilde{x}) = \frac{\tilde{x}}{d}$.\\
When optimizing the weights $\{\theta^l\}_{l=1}^{L_g},\{w^l\}_{l=1}^{L_f}$ instead, we use the recursive formulas for the NTK of MLPs in \cite{exact}, equipped with linear activations, and arrive at $\Theta(\xi,\tilde{\xi}) = L_g\frac{\xi^\top \tilde{\xi}}{d_0}, \mathcal{K}(x,\tilde{x}) = L_f\frac{x^\top \tilde{x}}{d}$ and $\Xi(x,\tilde{x}) = L_f\frac{\tilde{x}}{d}$, concluding the proof.
\end{proof}

% \section{Experimental Details}\label{sec:exp}
\section {Empirical support for implicit acceleration - setup and additional results}\label{sec:exp}
\paragraph{Stochastic Gradient Descent}
To perform empirical experiments, we resort to the discrete time version of \cref{eqn:eq1,eqn:eq2,eqn:eq3} which will allow simulating SGD on the loss in function space. Let $\xi_{i_s}$ denote the sample fed into the SGD algorithm at step $s$. The discrete evolution of $\F,g,J$ assuming a learning rate $\mu$ is given by:
\begin{align}
g_{t+1}(\xi) &= g_{0}(\xi) -\mu\sum_{s=1}^t\Theta(\xi,\xi_{i_s})J_{s,i_s}^\top\chi_{s,i_s} \label{eqn:d1}\\
J_{t+1}(x) &= J_{0}(x) -\mu\sum_{s=1}^t \chi_{s,i_s}\Xi(x,g_{s,i_s})^\top \label{eqn:d2}\\
\F_{t+1}(\xi) &= \F_{0}(\xi) -\mu\sum_{s=1}^t\mathcal{K}\big(g_{t+1}(\xi),g_{s,i_s}\big)I_{d_r}\chi_{s,i_s} \label{eqn:d3}
\end{align}

\subsection{Model description} 

We use 1 hidden layer MLPs with ReLU\footnote{While ReLUs dont meet the criteria of a twice differentiable activations, our results indeed hold still by approximating the ReLU function using the softplus function $\phi(a;m) = \frac{1}{m}\log(1 + \text{e}^{ma})$, and taking its limit $m \to \infty$ after taking $n \to \infty$.} activations for both $f,g$ in all experiments. For any pair of inputs $\xi,\tilde{\xi}$ with $D = \|\xi\|\|\tilde{\xi}\|, \lambda = \frac{\xi^\top \tilde{\xi}}{D}$, the NTK function $\Theta(\xi,\tilde{\xi})$ is given by $\Theta(\xi,\tilde{\xi}) = \frac{\xi^\top\tilde{\xi}}{d_0}\dot{\Sigma}(\xi,\tilde{\xi}) + \Sigma(\xi,\tilde{\xi})$
where:
\begin{align}
    \Sigma(\xi,\tilde{\xi}) = \frac{D}{d_0}\frac{\lambda(\pi - \text{arccos}(\lambda)) + \sqrt{1 - \lambda^2}}{2\pi},~~~
    \dot{\Sigma}(\xi,\tilde{\xi}) = \frac{\pi - \text{arccos}(\lambda)}{2\pi}
\end{align}
$\mathcal{K}(x,\tilde{x})$ computes the same function with the replacement $d_0 \to d$, and the inputs $D = \|x\|\|\tilde{x}\|, \lambda = \frac{x^\top \tilde{x}}{D}$.\\
The function $\Xi(x,\tilde{x})$ is computed by taking the derivative of $\mathcal{K}(x,\tilde{x})$ with respect to the first input $x$. This results in the following:
\begin{align}
    \Xi(x,\tilde{x}) = \frac{\tilde{x}}{d}\dot{\Sigma} + \frac{x^\top \tilde{x}}{d}\frac{1}{2\pi\sqrt{1 - \lambda^2}}\big(\frac{\tilde{x}}{D} - \lambda \frac{x}{\|x\|^2}\big) + \frac{D}{d}\dot{\Sigma}(x,\tilde{x})\big(\frac{\tilde{x}}{D} - \lambda \frac{x}{\|x\|^2}\big) + \frac{x\|\tilde{x}\|}{\|x\|D}\Sigma(x,\tilde{x})
\end{align}
We consider a range of bottleneck widths from infinite bottleneck (``no bottleneck'') (labeled as inf in plots) to a maximum value that varies per dataset as noted in the result section below. Finally, we use initialized networks with $n = 10000$ to sample $g_{0}(\xi), J_0(x), F_0(\xi)$ for all inputs. Experiments are implemented using the Tensorflow \cite{tensorflow2015-whitepaper} package.

\subsection{Tasks and datasets}

We consider a regression problem on synthetic dataset in \cref{expt:data:synthetic}, multi-class classification on MNIST dataset \cite{mnist} in \cref{expt:data:mnist} and a binary classification problem (dog vs deer) on CIFAR-10 dataset \cite{cifar} in \cref{expt:data:cifar10}. The networks are optimized with mean squared error (MSE) criterion where the targets are one-hot encoded vectors in classification experiments. 

\subsection{Compute budget}
We train all models for a fixed amount of time per set of hyperparameters that include learning rate and mini-batch size per dataset. This budget is necessary in order for us to manage the burden imposed on our compute resources. Consequently narrower finite width bottleneck models train for a larger number of steps compared to wider finite width bottleneck models.

\subsection{Results}
\label{appendix:expt:results}
We present the results of experiments with infinite width bottleneck nonlinear models in this section. Our main finding is that the acceleration effect suggested by \cref{lemma:linear} for linear models during optimization is also seen with nonlinear models on both synthetic and real datasets. Our presentation includes two types of plots that allows us to make observations about both training speed as well as test performance:
\paragraph{Loss evolution} We present plots that show the evolution of loss and accuracy for both training and test datasets. These plots allow us to make observations on the implicit acceleration effect in infinite width bottleneck models as well as its impact on test performance.

\paragraph{Training vs test loss} \label{appendix:expt:results:tvt} Additionally, we plot test loss as a function of training loss for several bottleneck widths. This plot allows us to compare test performance for a given training loss across different bottleneck widths. 

The procedure used to generate this plot is described below for completeness:
\begin{itemize}
    \item determine a range for training loss values that covers all bottleneck widths for a given dataset. We set this range to be the minimum and maximum recorded training losses across all bottleneck widths for each dataset.
    \item resample training loss values in the above range from training loss values recorded in our experiments.
    \item determine test loss values for the training loss values obtained in the above step again by resampling with linear interpolation.
    \item smooth both training and test loss values with a moving average filter of length 5.
\end{itemize}

\subsubsection{Synthetic dataset}
\label{expt:data:synthetic}

\begin{figure}[htb]
\centering
\begin{tabular}{cc}
\includegraphics[width=.5\columnwidth]{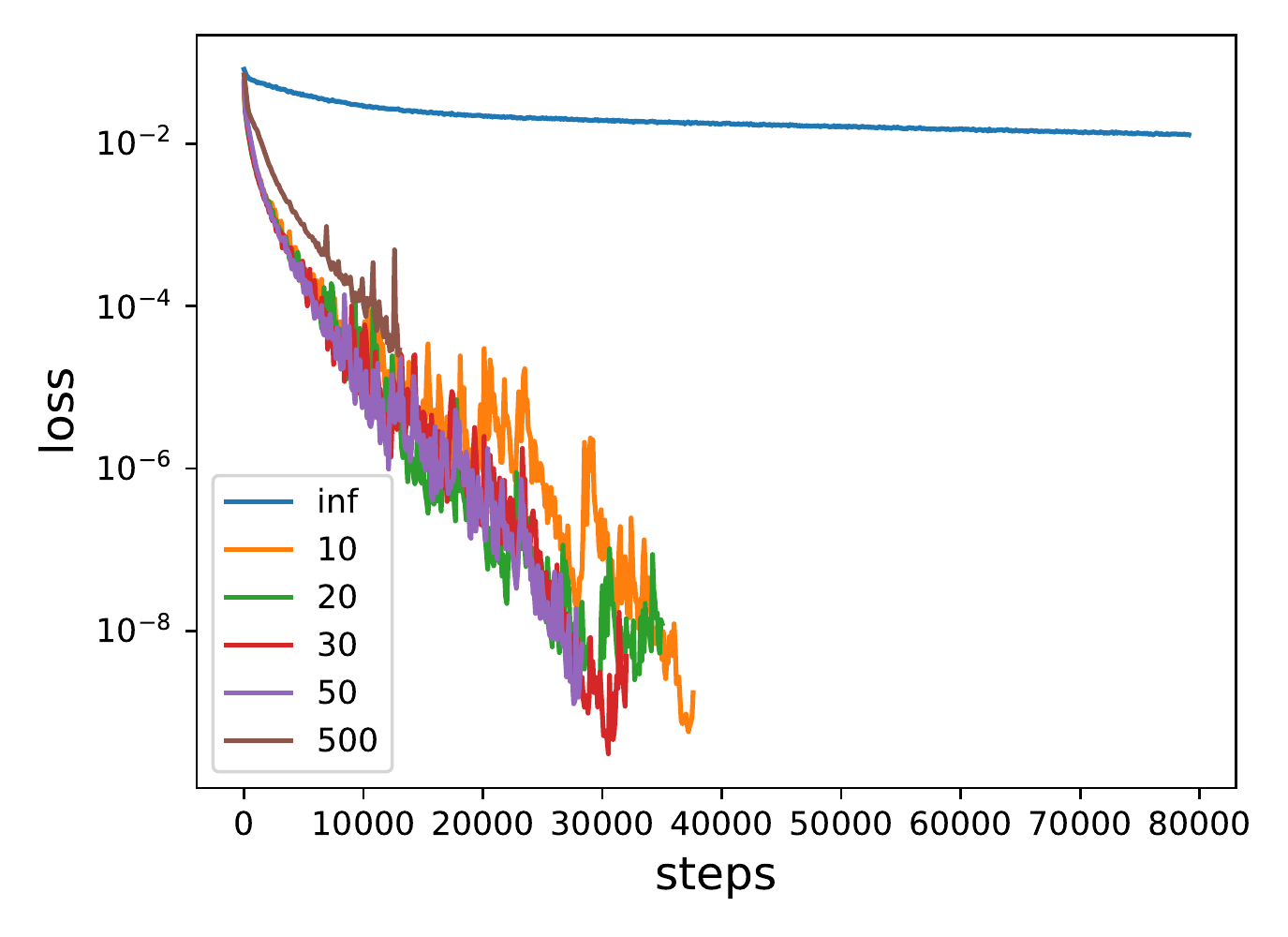}&
\includegraphics[width=.5\columnwidth]{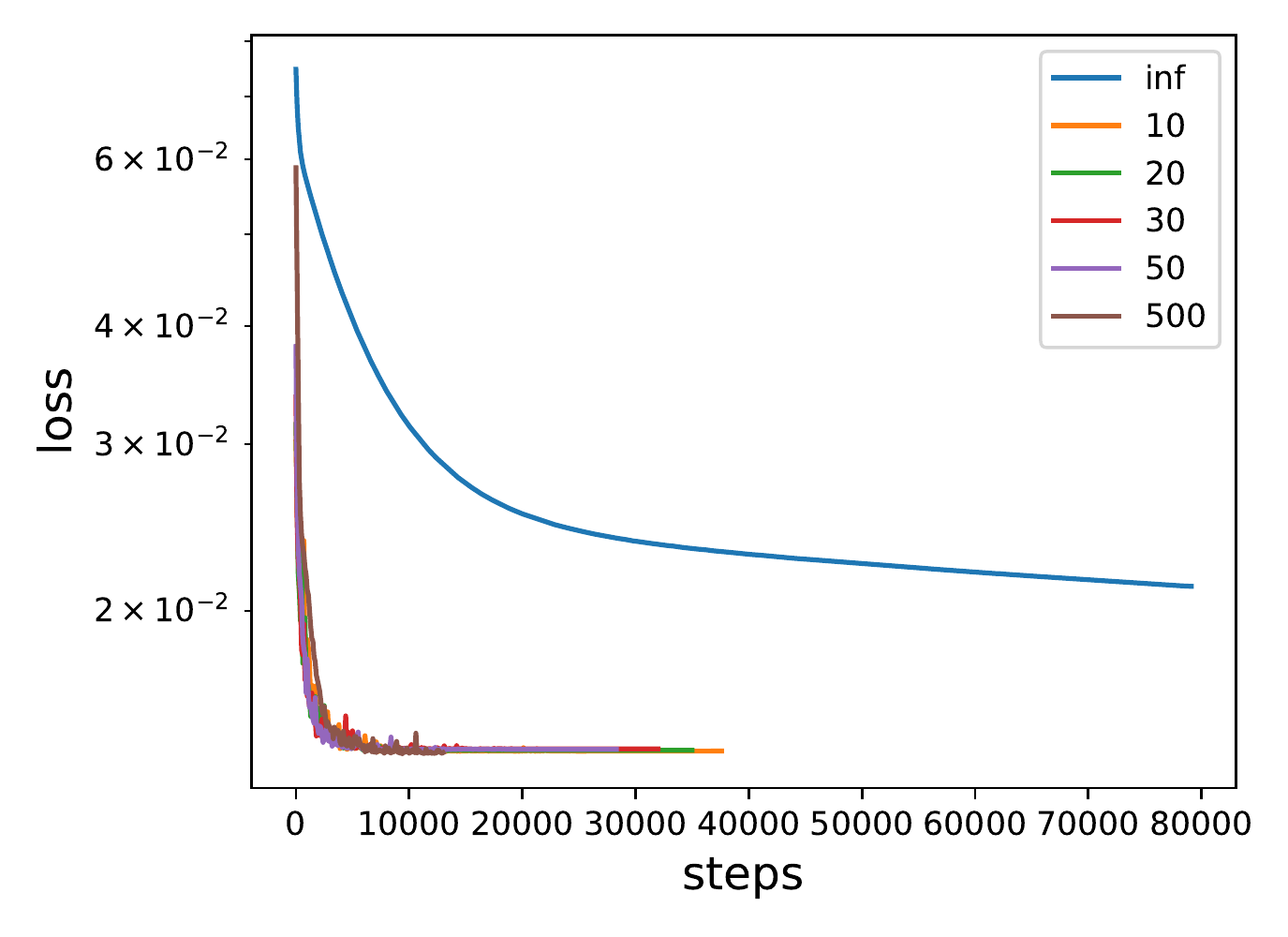}\\
(a) & (b)\\
\end{tabular}
  \caption{{Synthetic dataset:  metrics for finite bottleneck models of various widths. Widths are shown from smallest to largest where inf indicates infinite bottleneck (a) training loss  (b) test loss}}%
  \label{fig:synth_lr2k_metrics}
  %\vspace{-.2cm}
\end{figure}

\begin{figure}[htb]
\centering
\begin{tabular}{cc}
\includegraphics[width=.5\columnwidth]{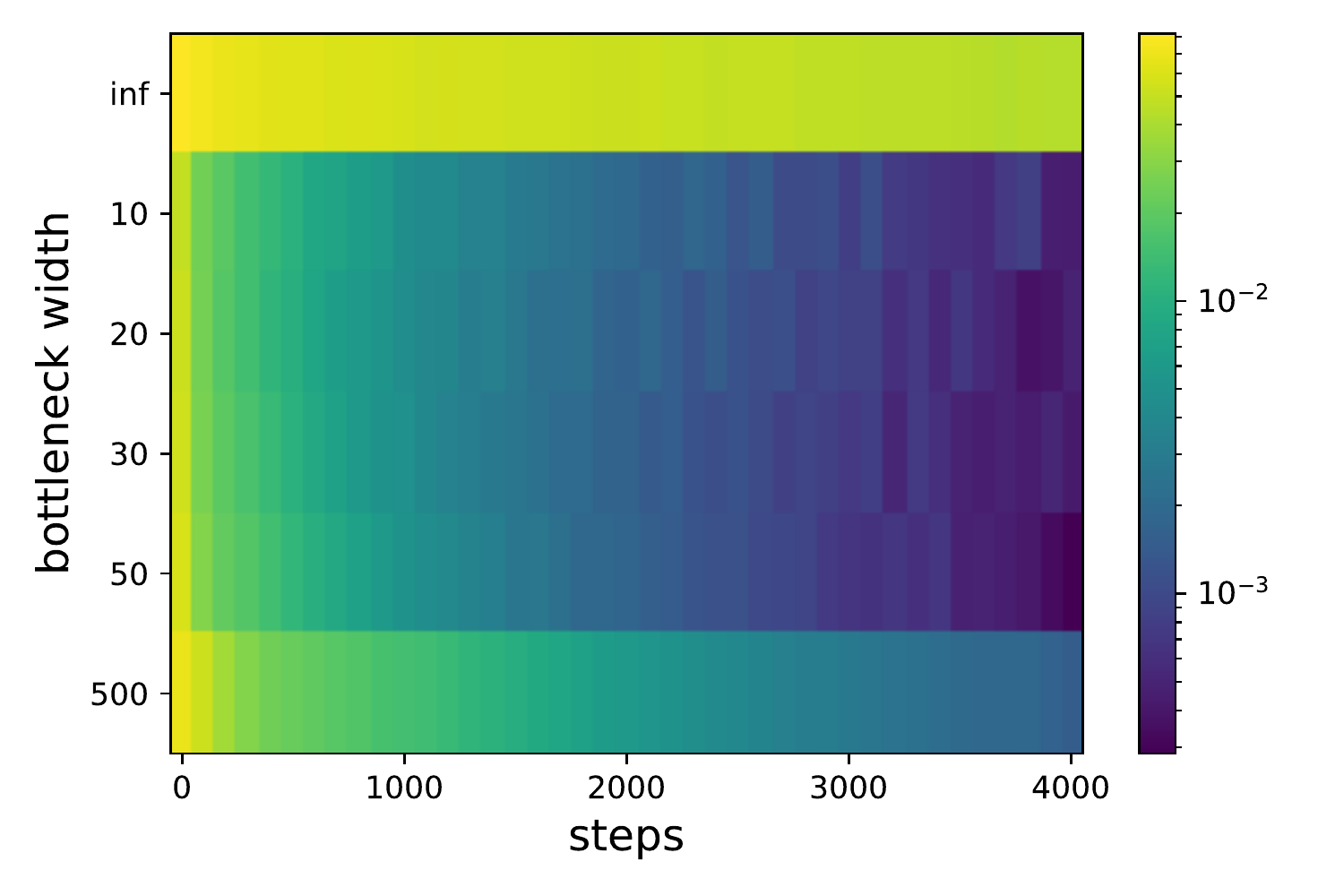}&
\includegraphics[width=.5\columnwidth]{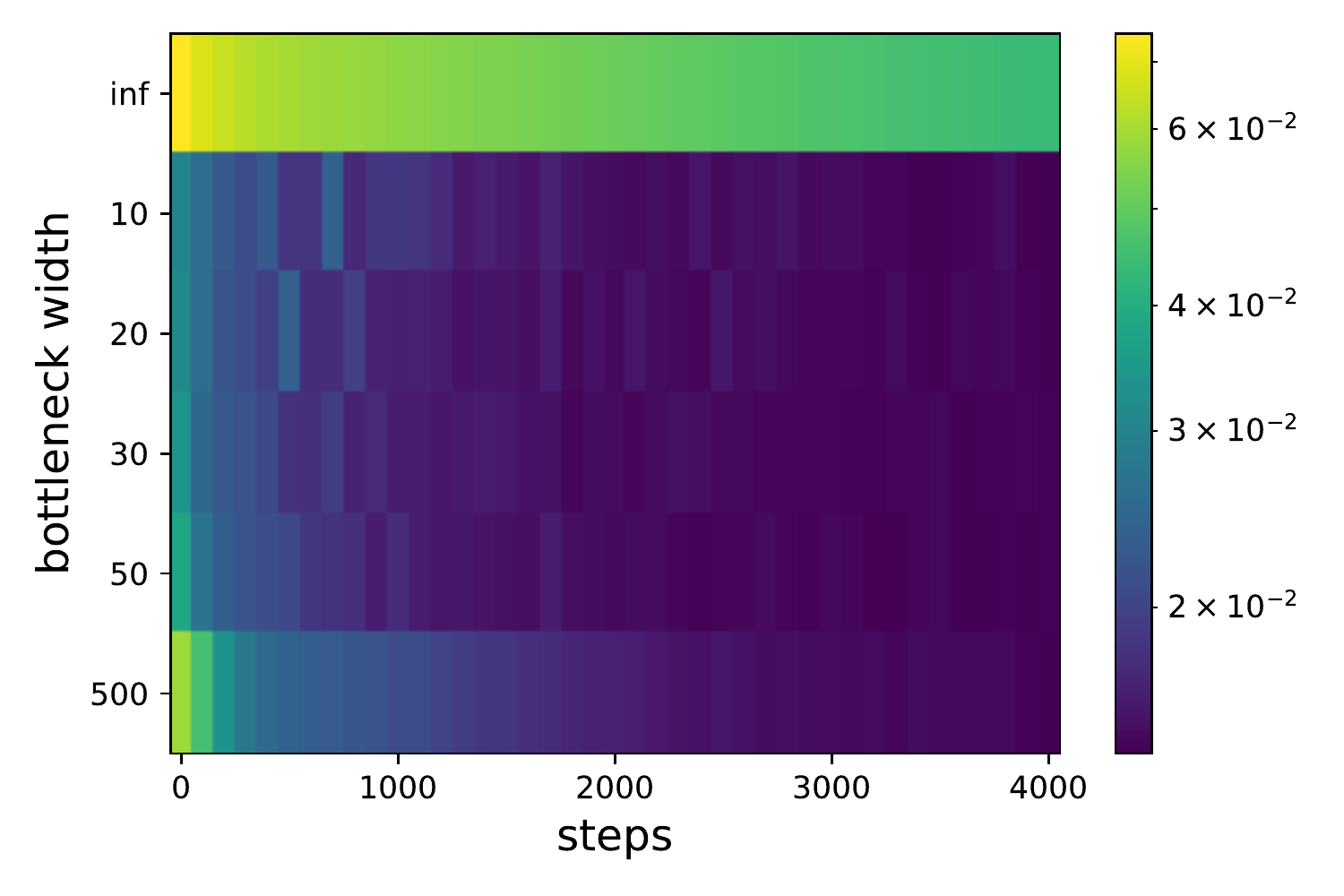}\\
(a) & (b)\\
\end{tabular}
  \caption{{Synthetic dataset: heatmap view of (a) training loss evolution and (b) test loss evolution for the first 4000 steps}}%
  \label{fig:synth_lr2k_heatmap}
  %\vspace{-.2cm}
\end{figure}

% \begin{figure}[!b]
% \centering
% \includegraphics[width=.5\columnwidth]{workshop_figures/synth/train_vs_test.pdf}\\
%   \caption{{Synthetic dataset: training loss vs test loss plot}}%
%   \label{fig:synth_lr2k_tvt}
%   %\vspace{-.2cm}
% \end{figure}
\begin{figure}[!b]
\centering
\begin{tabular}{cc}
\includegraphics[width=.5\columnwidth]{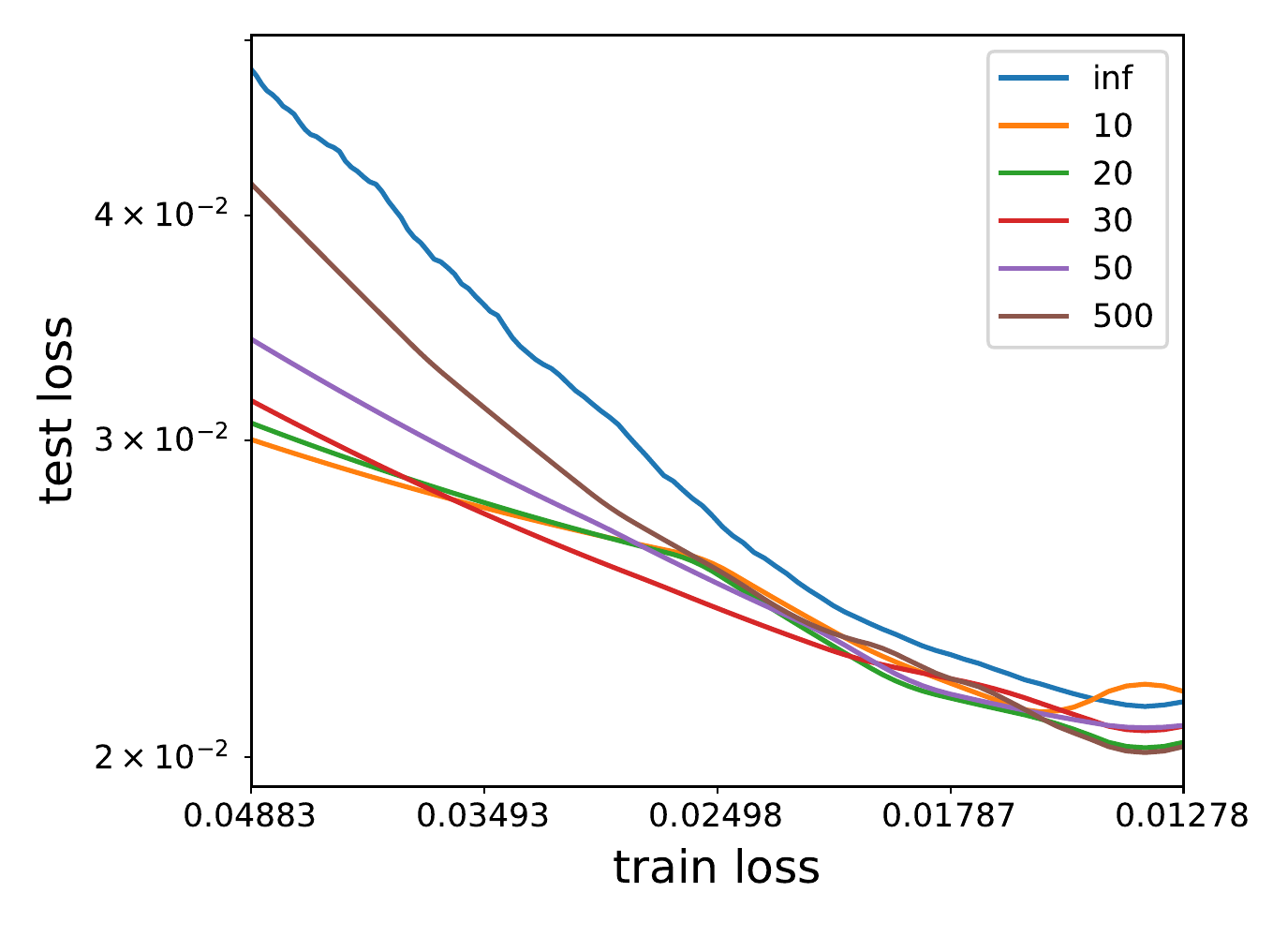}&
\includegraphics[width=.5\columnwidth]{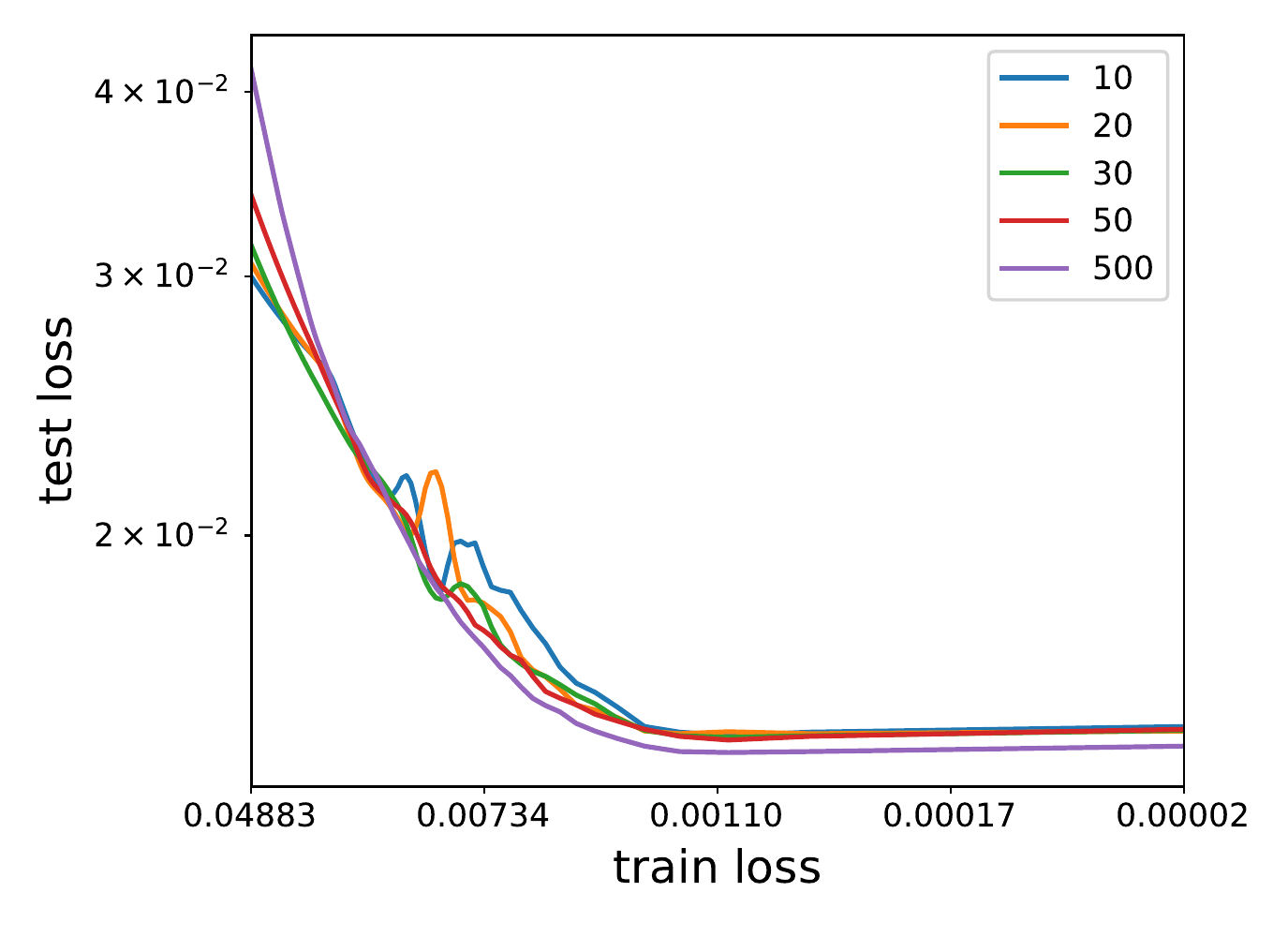}\\
(a) & (b)\\
\end{tabular}
  \caption{{Synthetic dataset: training loss vs test loss plot for (a) finite and infinite width bottleneck models and (b) finite width bottleneck models}}%
  \label{fig:synth_lr2k_tvt}
  %\vspace{-.2cm}
\end{figure}

We first study the behavior of infinite width bottleneck models by running \cref{eqn:d1,eqn:d2,eqn:d3} on a regression problem with synthetic data. To this end, we generate a dataset that consists of 2000 training and test samples where the:
\begin{itemize}
    \item inputs are random vectors sampled from a standard multivariate normal distribution.
    \item targets are formed by passing these vectors through a deep Gaussian Process model.
\end{itemize}
 The widths considered for this dataset range from infinite width (inf) to a width of 500. The models are trained with SGD using a mini-batch size of 20 and learning rate set to 2000.

\paragraph{Loss evolution}
Figure~\ref{fig:synth_lr2k_metrics}a and figure~\ref{fig:synth_lr2k_metrics}b show the evolution of training and test loss for models with different bottleneck widths. Figure~\ref{fig:synth_lr2k_metrics}a clearly shows that architectures with narrower bottleneck train faster than their wider bottleneck counterparts. Figure~\ref{fig:synth_lr2k_heatmap} shows a heatmap view of the evolution of training and test loss for the first 4000 steps to zero in on the early phase of learning. Figure~\ref{fig:synth_lr2k_heatmap} clearly demonstrates the acceleration effect of training and test loss for small to intermediate bottleneck sizes, while no bottleneck and large bottleneck models exhibit slower learning.

\paragraph{Training vs test loss}
We observe from figure~\ref{fig:synth_lr2k_metrics}a that the infinite width bottleneck model exhibits a higher training loss compared to its finite width bottleneck counterparts. To make meaningful comparisons, we resort to using two plots for this dataset --- figure~\ref{fig:synth_lr2k_tvt}a shows performance of all models including infinite width bottleneck (labeled as inf) and figure~\ref{fig:synth_lr2k_tvt}b that shows performance of finite width bottlenecks only. We make the following observations from the two plots:
\begin{itemize}
    \item finite width bottleneck models lead to a lower test loss for the same training loss over infinite width bottlenecks models.
\item the widest bottleneck width model (500) has the best test performance as seen from figure~\ref{fig:synth_lr2k_tvt}b
\end{itemize}
The above observations suggest that using a bottleneck in infinite width models can lead to better test performance across a range of bottleneck widths over infinite width (``no bottleneck'') model. The narrowest bottleneck width (10) considered here in this experiment appears to show an increase in test loss. This, however, is an artifact that arises due the training loss range displayed in figure~\ref{fig:synth_lr2k_tvt}a. We see from figure~\ref{fig:synth_lr2k_tvt}b the the training vs test loss curve eventually decreases and matches values observed in models with bottleneck widths of 20, 30 and 50. This non-monotonic behavior may be caused by sub-optimal hyperparameter selection that we plan to address as future work.

\subsubsection{MNIST dataset}
\begin{figure*}[htb]
    %\centering % Not needed
    \begin{subfigure}[b]{0.25\textwidth}
        \includegraphics[width=\textwidth]{workshop_figures/MNIST/train_loss.pdf}
        \caption{train loss}
        \label{fig:mnist_lr5k_metrics:train_loss}
    \end{subfigure}\hfill
    \begin{subfigure}[b]{0.25\textwidth}
        \includegraphics[width=\textwidth]{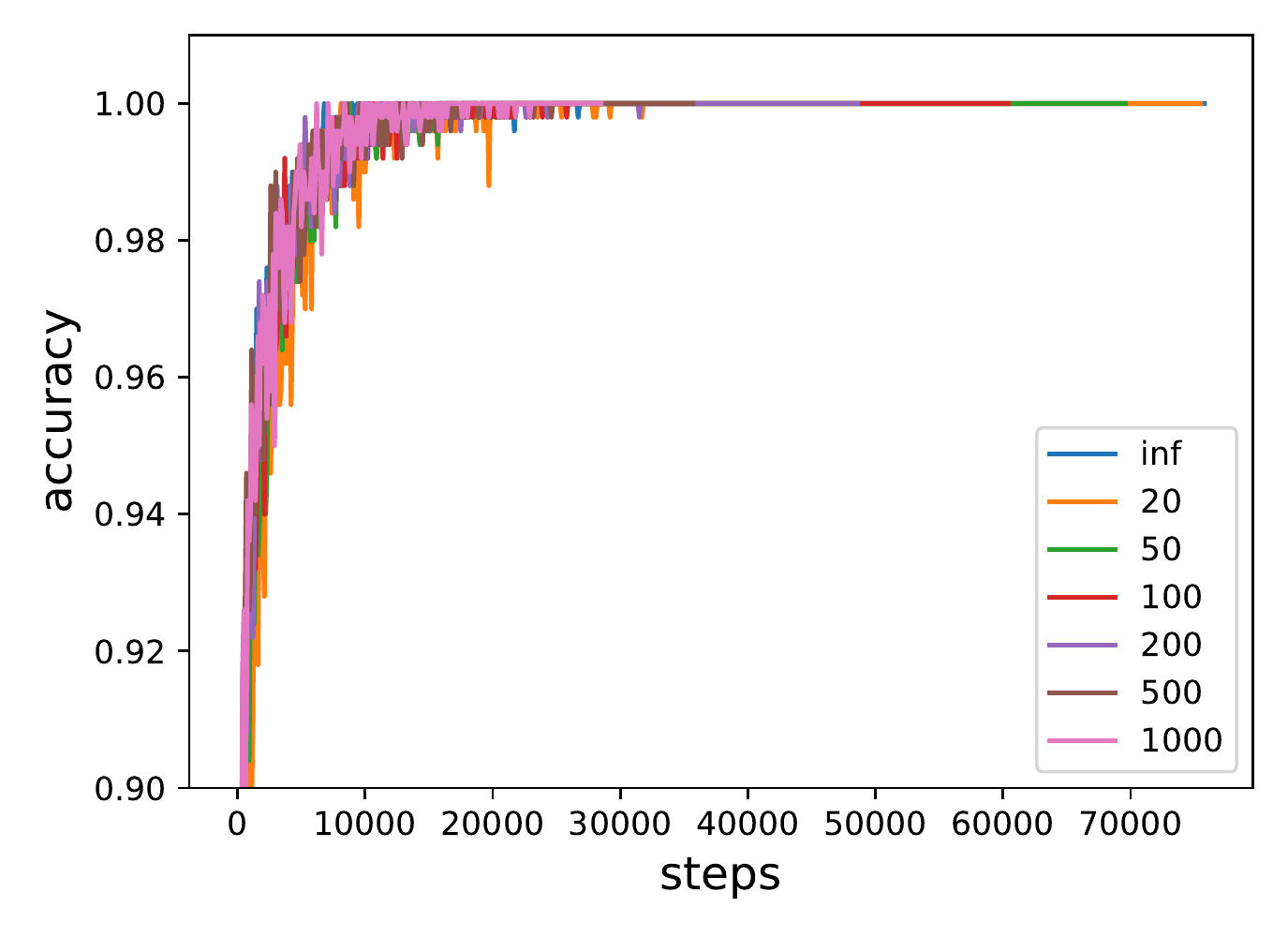}
        \caption{train accuracy}
        \label{fig:mnist_lr5k_metrics:train_acc}
    \end{subfigure}\hfill
    \begin{subfigure}[b]{0.25\textwidth}
        \includegraphics[width=\textwidth]{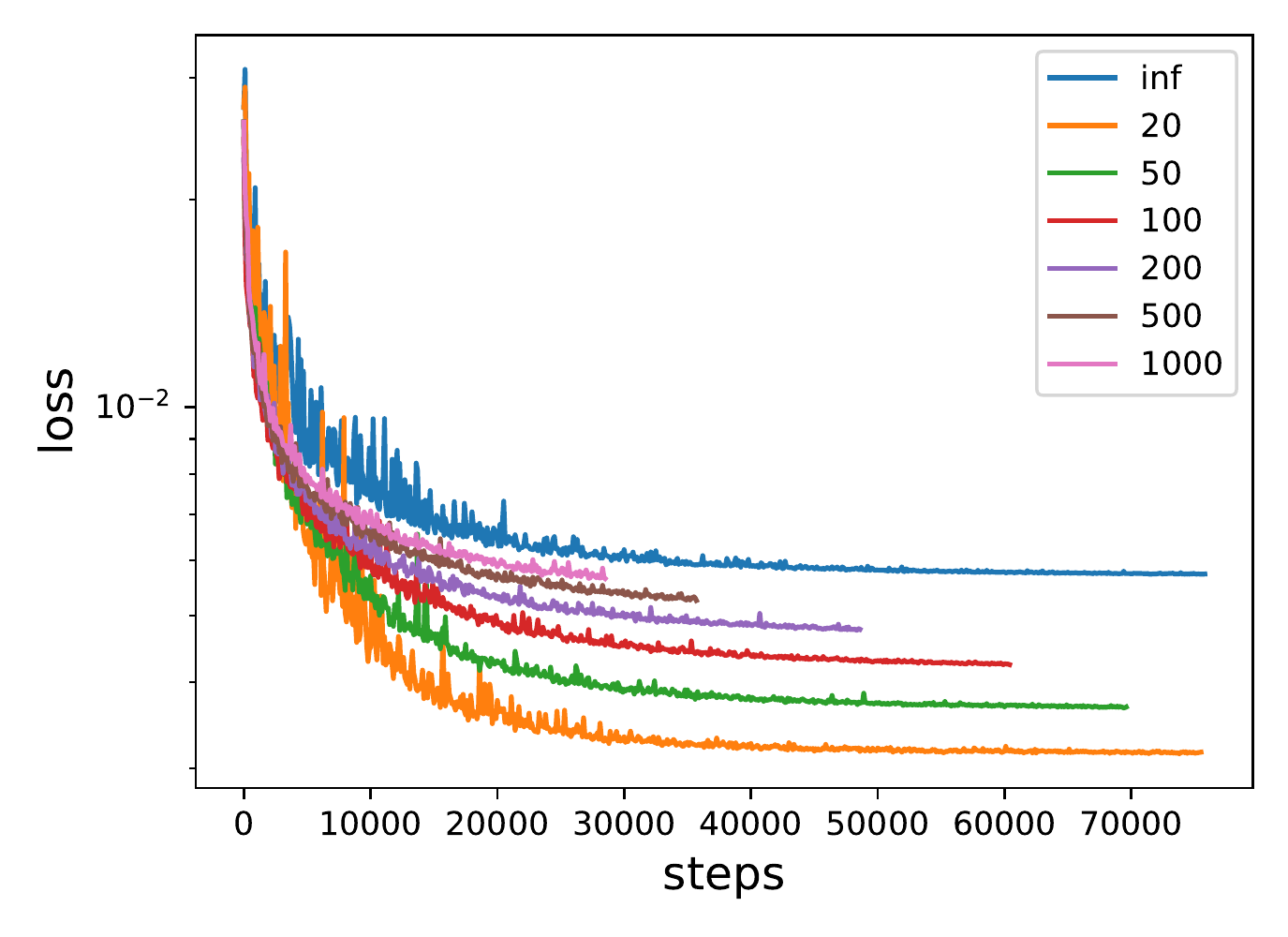}
        \caption{test loss}
        \label{fig:mnist_lr5k_metrics:test_loss}
    \end{subfigure}\hfill
    \begin{subfigure}[b]{0.25\textwidth}
        \includegraphics[width=\textwidth]{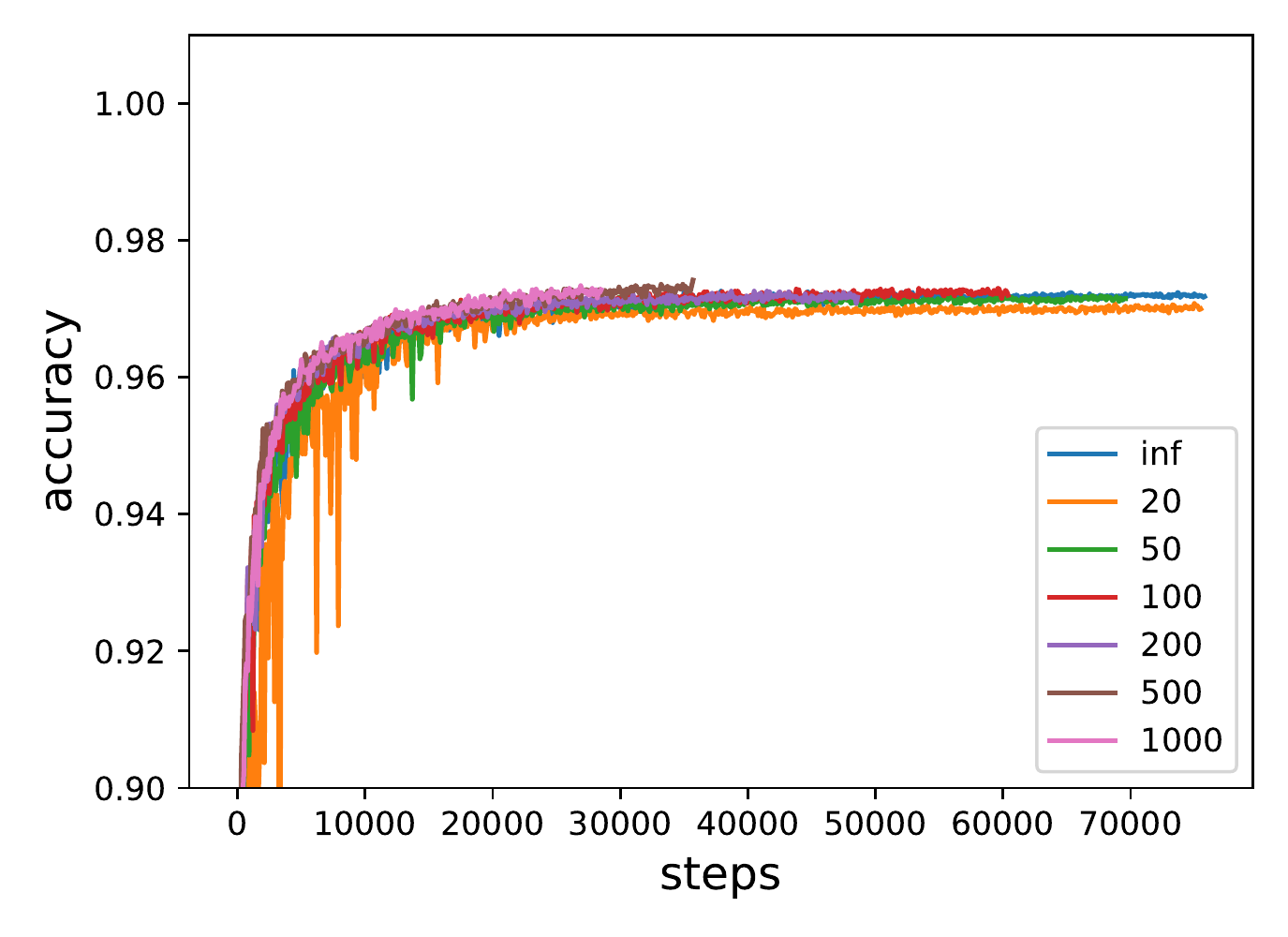}
        \caption{test accuracy}
        \label{fig:mnist_lr5k_metrics:test_acc}
    \end{subfigure}\hfill
    \caption{{MNIST dataset: training metrics for bottleneck models with widths from smallest to largest (inf indicates infinite width bottleneck)}}%
    \label{fig:mnist_lr5k_metrics}
\end{figure*}
We study the behavior of infinite width neural networks on MNIST dataset \cite{mnist} in this section. We train models on a multi-class (10-class) classification problem using the MSE loss in this experiment. The bottleneck widths used here range from ``no bottleneck'' (labeled as inf) to 1000 and are trained with SGD using a learning rate of 250. As with the previous section, we discuss implicit acceleration during optimization followed by test performance below:
\label{expt:data:mnist}
\paragraph{Loss and accuracy evolution} Figure~\ref{fig:mnist_lr5k_metrics} shows the evolution of training and test loss and accuracy over the course of training. We observe that training loss decays rapidly for networks with narrower bottlenecks. Figure~\ref{fig:mnist_lr5k_heatmap} (\cref{sec:exp}) further examines training and test loss evolution for the first 4000 steps, confirming that training accelerates as bottleneck size narrows. Note that while narrower bottleneck models exhibit lower test loss compared to their wider bottleneck counterparts, this does not hold for test accuracy, potentially due to a loss-metric mismatch as is commonly observed in practice \cite{pmlr-v97-huang19f}.
These results are inline with the analysis of linear networks with bottlenecks, that wide networks with bottlenecks accelerate training, and confirm empirically that this effect holds in practical nonlinear settings.
\begin{figure}[htb]
\centering
\begin{tabular}{cc}
\includegraphics[width=.5\columnwidth]{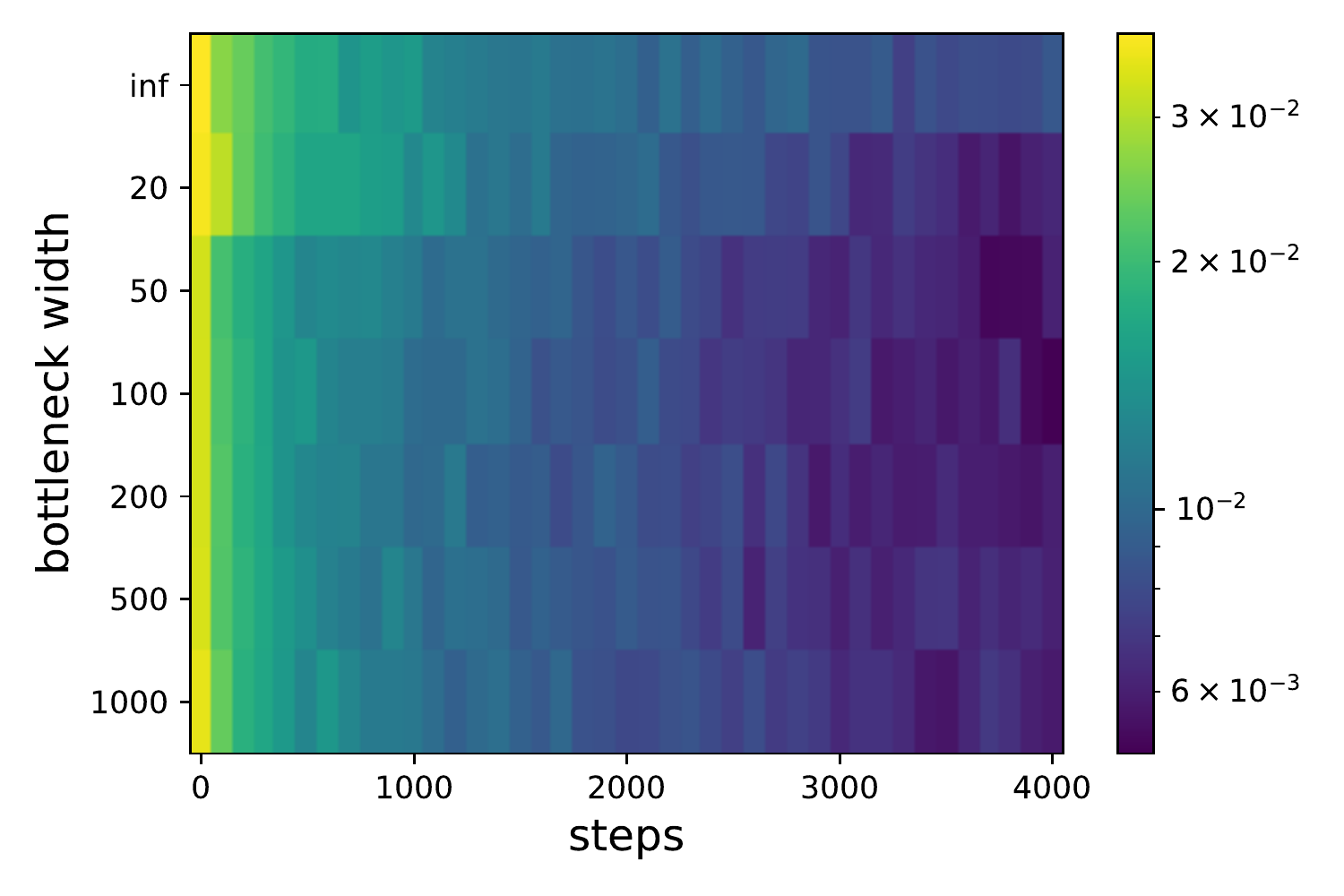}&
\includegraphics[width=.5\columnwidth]{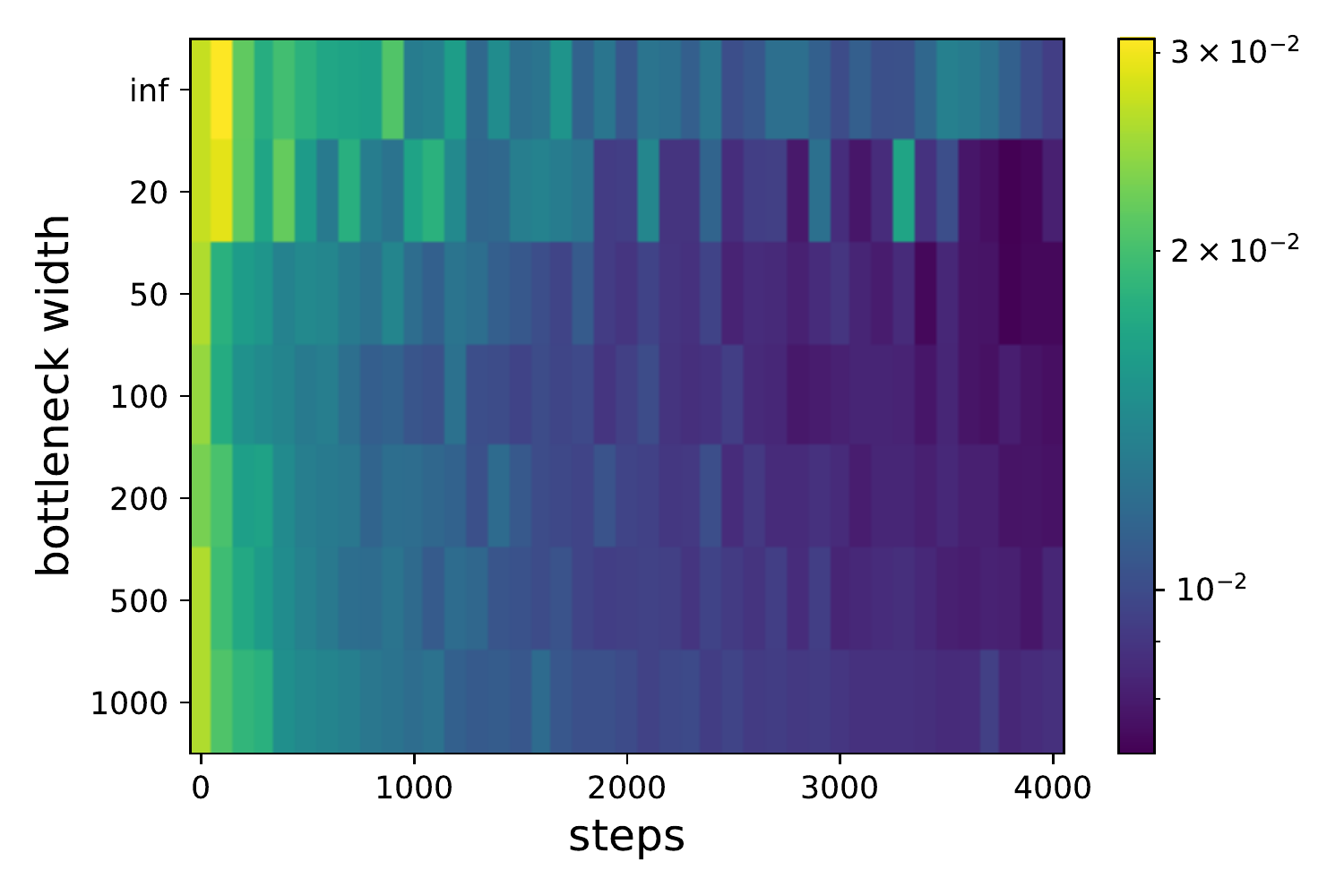}\\
(a) & (b)\\
\end{tabular}
  \caption{{MNIST dataset: heatmap view of (a) training loss evolution and (b) test loss evolution for the first 4000 steps}}%
  \label{fig:mnist_lr5k_heatmap}
  %\vspace{-.2cm}
\end{figure}

\begin{figure}[htb]
\centering
\includegraphics[width=0.5\columnwidth]{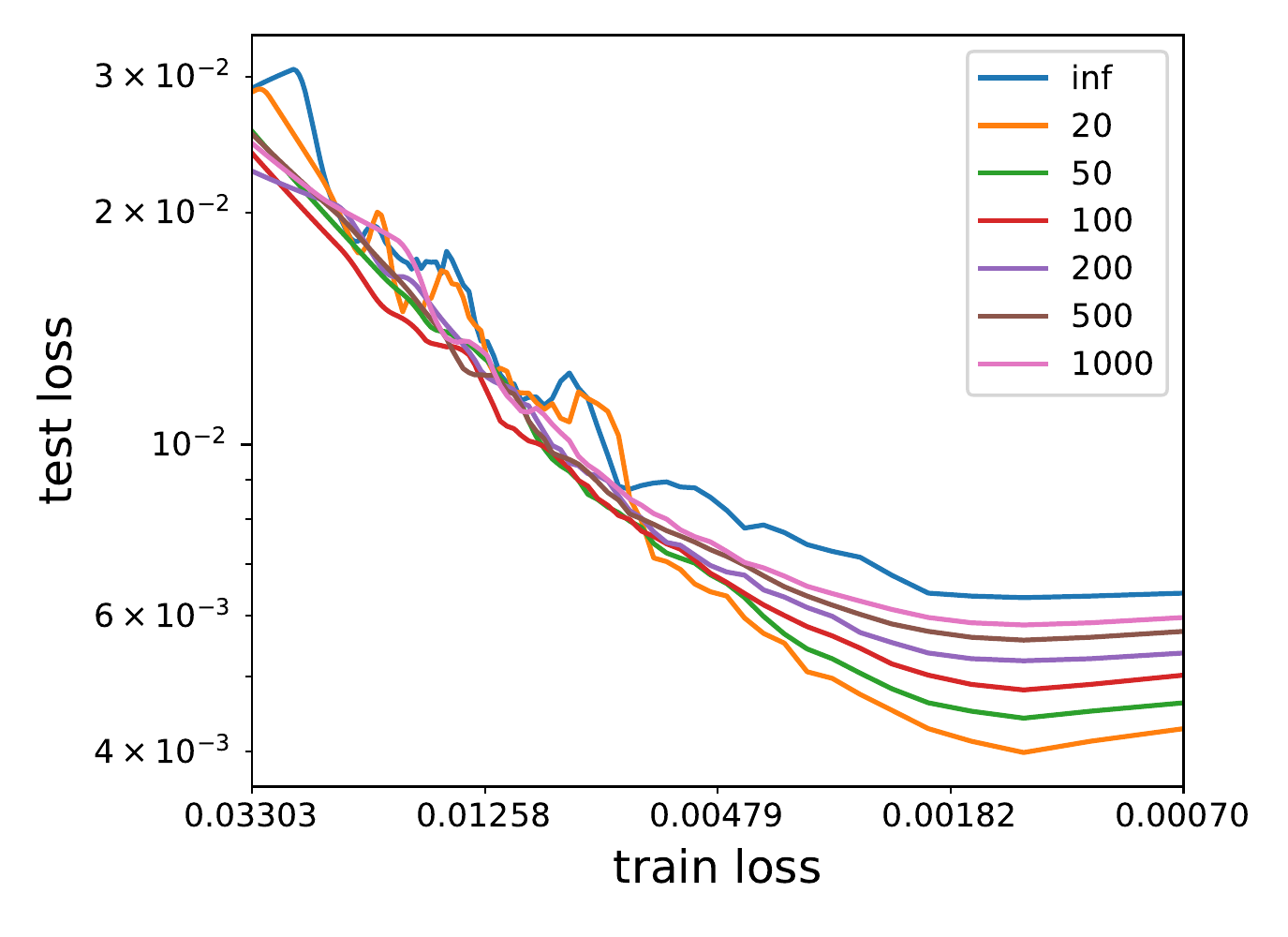}\\
  \caption{{MNIST dataset: training loss vs test loss plot}}%
  \label{fig:mnist_lr5k_tvt}
  %\vspace{-.2cm}
\end{figure}

\paragraph{Training vs test loss} Figure~\ref{fig:mnist_lr5k_tvt} shows a plot of training vs test loss for MNIST dataset from which we make the following observations:
\begin{itemize}
    \item finite width bottleneck models achieve lower test loss for a given training loss over infinite width bottlenecks.
    \item narrower bottlenecks attain lower test loss compared to their wider finite width counterparts for all widths considered for this dataset.
\end{itemize}
The above observations suggest that using bottlenecks in infinite width nonlinear models can lead to a lower test loss and hence improved test performance on real datasets as well. However, we repeat our previous observation that lower test loss does not always lead to an improvement in test accuracy. 

\subsubsection{CIFAR-10 dataset}
Finally, we study minimizing loss in function space on CIFAR-10 \cite{cifar} which is another instance of real world data. We study a binary classification problem using two classes from CIFAR-10 in this set of experiments. As before, we first present our findings on acceleration effect during training followed by observations on test performance. The models are trained with SGD using a learning rate of 1000 in these experiments.
\label{expt:data:cifar10}
\paragraph{Loss and accuracy evolution} Figure~\ref{fig:cifar10_lr1k_metrics} shows training loss, training accuracy, test loss, test accuracy for experiments on CIFAR10 dataset.  Figure~\ref{fig:cifar10_lr1k_heatmap} shows the evolution of training loss for 4000 steps and clearly illustrates the acceleration effect in finite bottleneck models. The accelerated learning effect is stronger in narrower bottleneck architectures. Interestingly, we observe that lower loss value leads to better test accuracy in this dataset unlike our observation for MNIST in \cref{expt:data:mnist}.

\paragraph{Training vs test loss} We make the following observations from figure~\ref{fig:cifar10_lr1k_tvt} that shows a plot of training vs test loss for CIFAR-10 dataset:
\begin{itemize}
    \item Larger finite width bottleneck models attain a lower test loss compared to infinite width bottleneck model.
    \item Larger finite width bottleneck models achieve a lower test loss compared to their narrower width counterparts.
\end{itemize}
\begin{figure*}
    %\centering % Not needed
    \begin{subfigure}[b]{0.25\textwidth}
        \includegraphics[width=\textwidth]{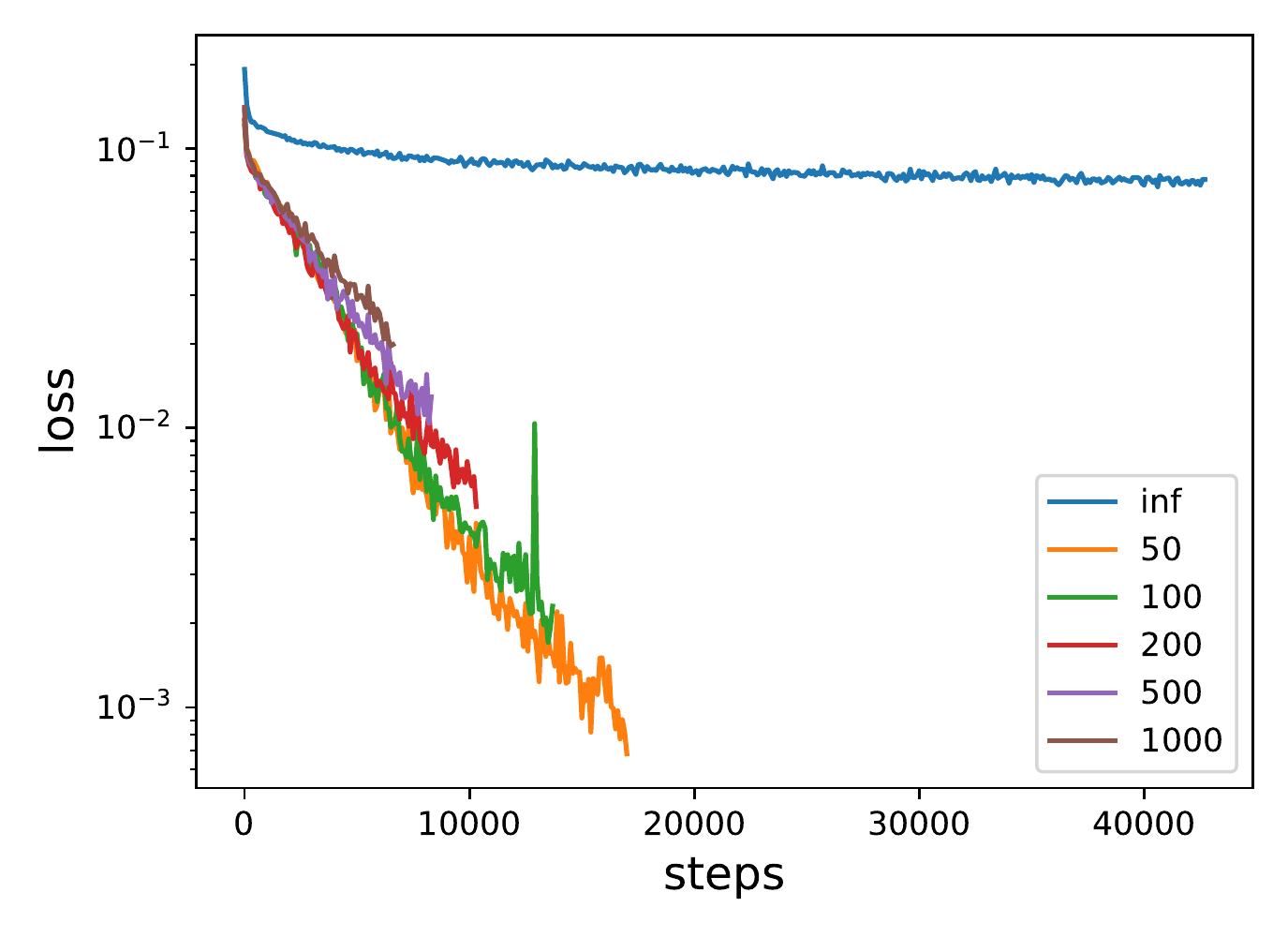}
        \caption{train loss}
        \label{fig:cifar10_lr1k_metrics:train_loss}
    \end{subfigure}\hfill
    \begin{subfigure}[b]{0.25\textwidth}
        \includegraphics[width=\textwidth]{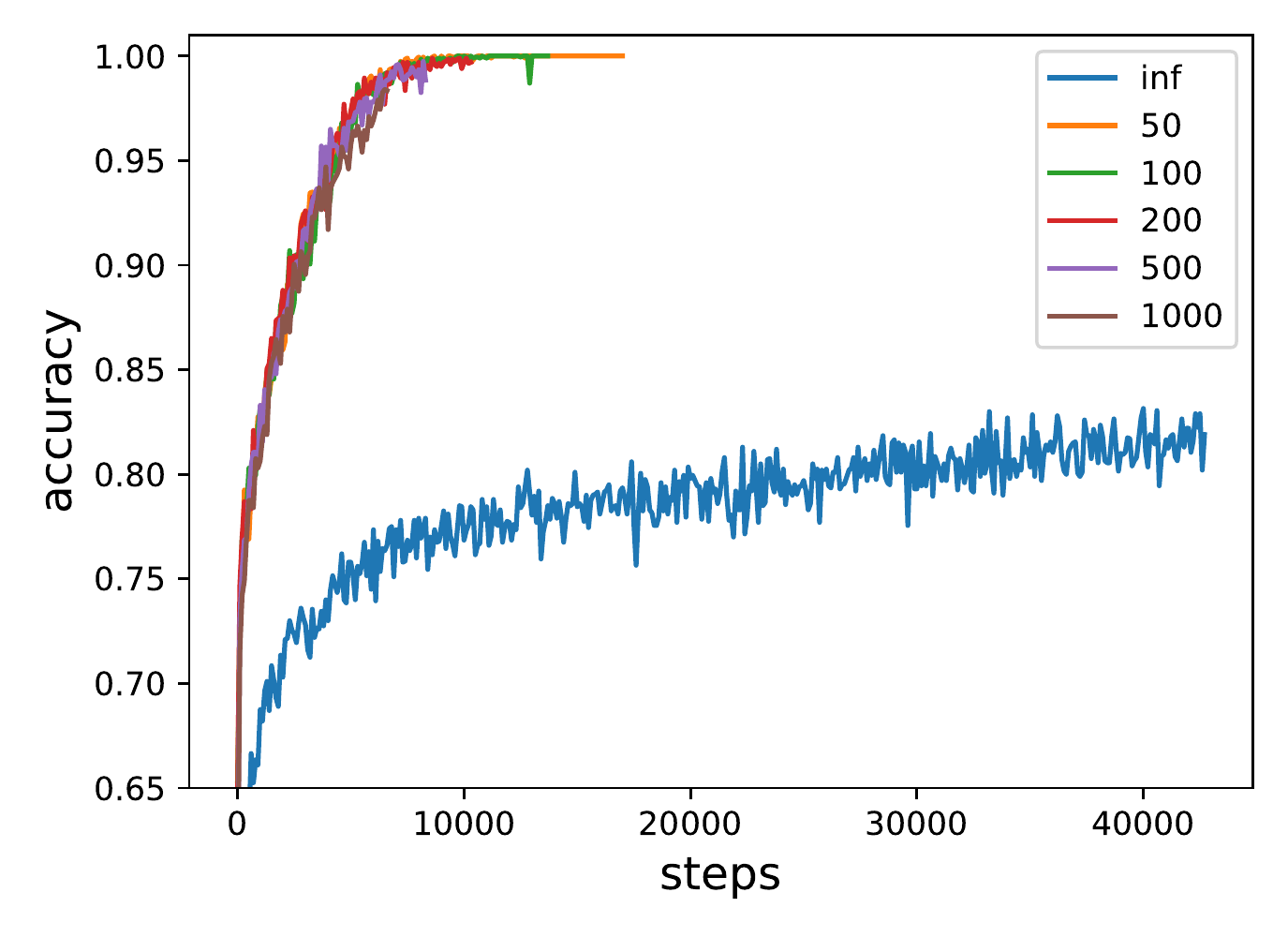}
        \caption{train accuracy}
        \label{fig:cifar10_lr1k_metrics:train_acc}
    \end{subfigure}\hfill
    \begin{subfigure}[b]{0.25\textwidth}
        \includegraphics[width=\textwidth]{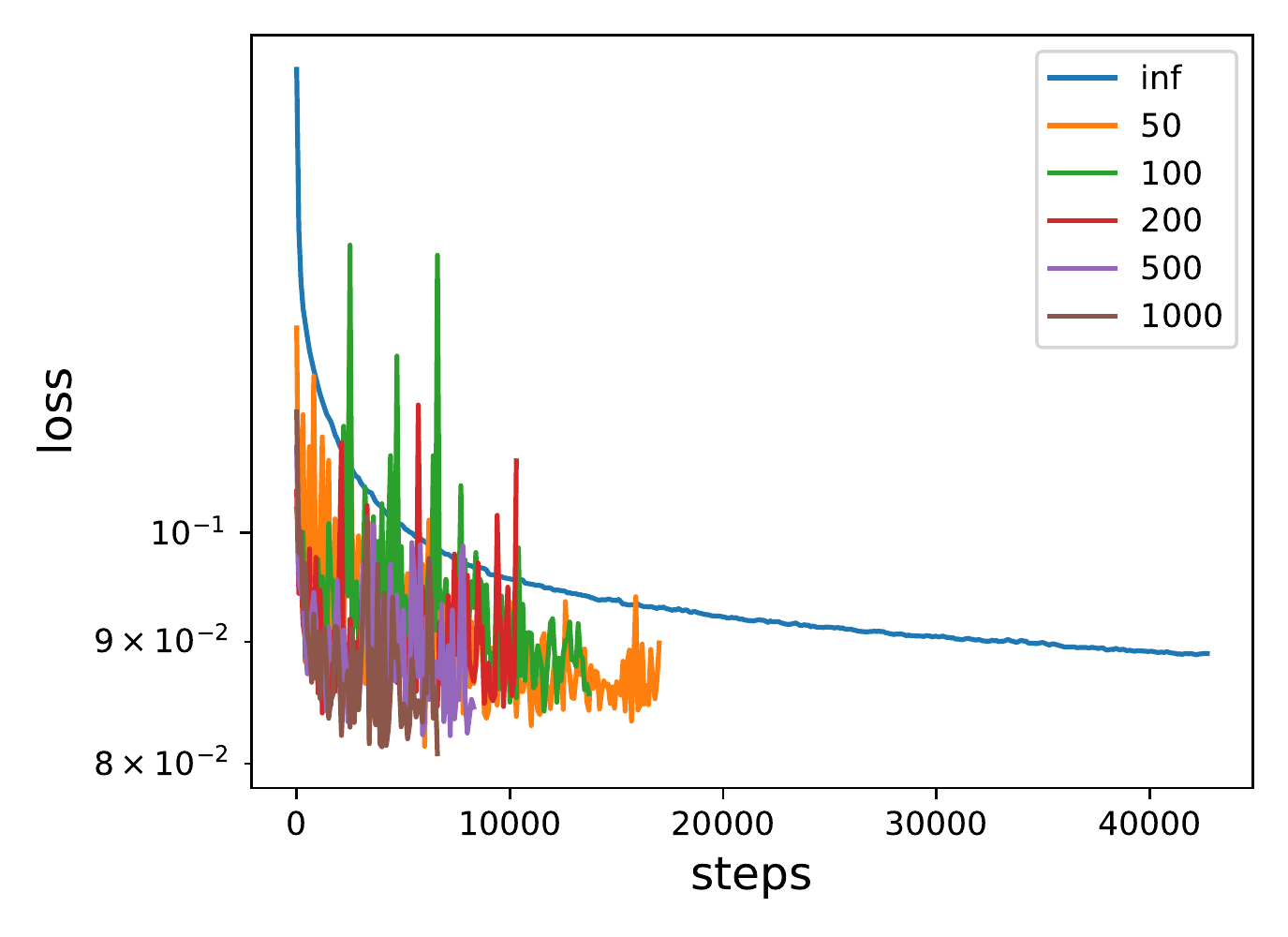}
        \caption{test loss}
        \label{fig:cifar10_lr1k_metrics:test_loss}
    \end{subfigure}\hfill
    \begin{subfigure}[b]{0.25\textwidth}
        \includegraphics[width=\textwidth]{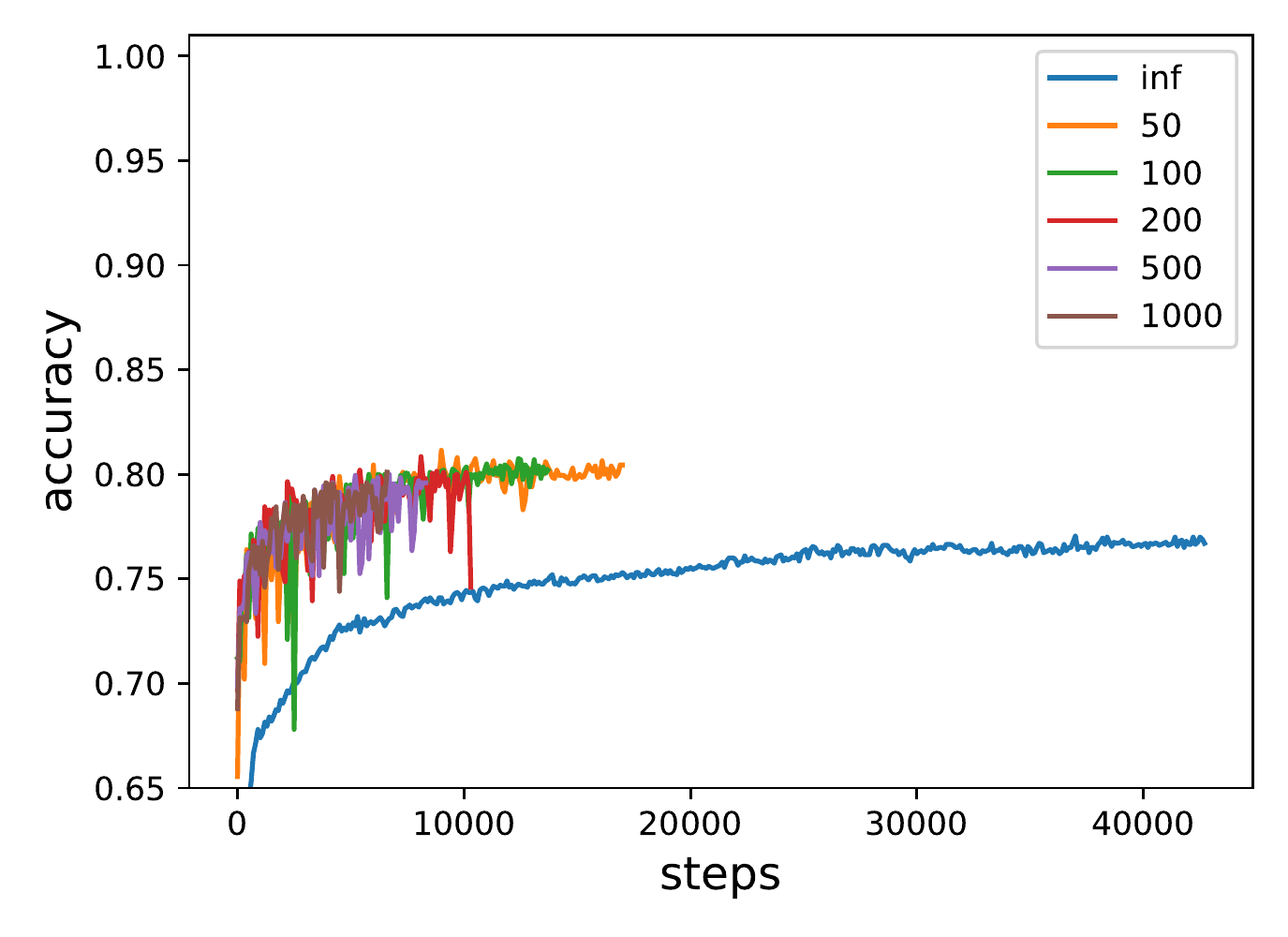}
        \caption{test accuracy}
        \label{fig:cifar10_lr1k_metrics:test_acc}
    \end{subfigure}\hfill
    \caption{{CIFAR-10 dataset: training metrics for bottleneck models with widths from smallest to largest (inf indicates infinite width bottleneck)}}%
    \label{fig:cifar10_lr1k_metrics}
\end{figure*}

\begin{figure}[ht]
\centering
\begin{tabular}{cc}
\includegraphics[width=.5\columnwidth]{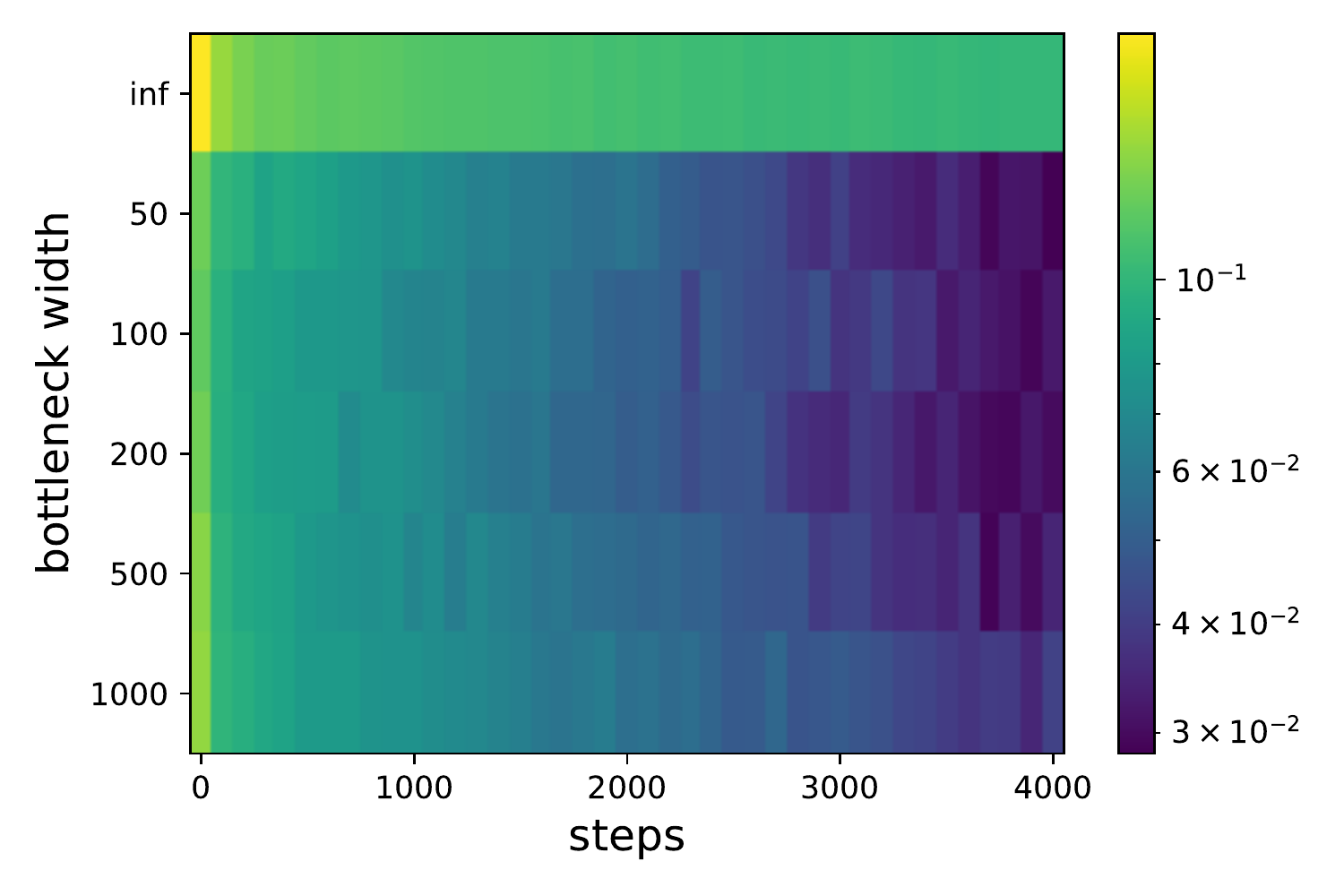}&
\includegraphics[width=.5\columnwidth]{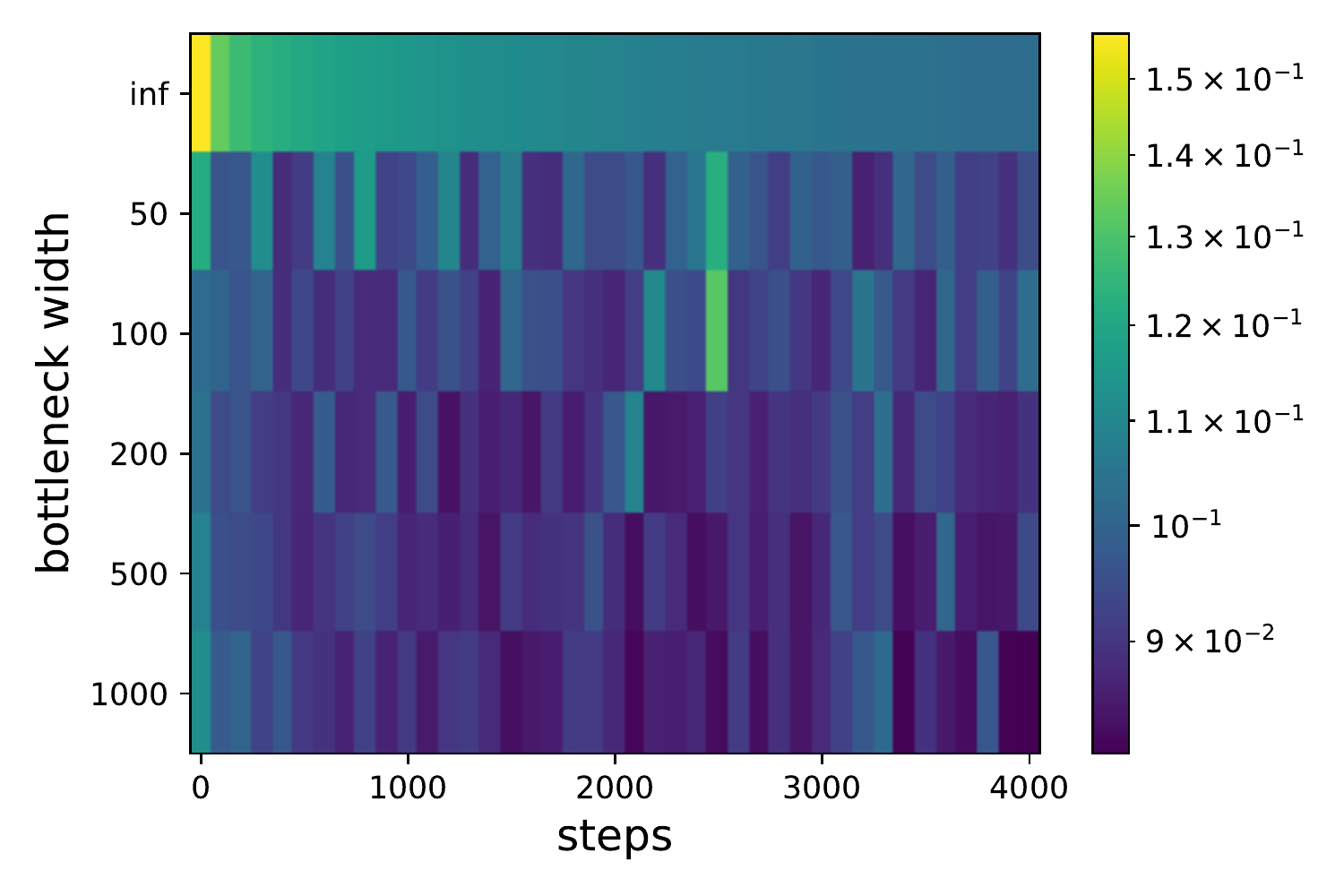}\\
(a) & (b)\\
\end{tabular}
  \caption{{CIFAR-10 dataset: heatmap view of (a) training loss evolution and (b) test loss evolution for the first 4000 steps}}%
  \label{fig:cifar10_lr1k_heatmap}
  %\vspace{-.2cm}
\end{figure}

\begin{figure}[!t]
\centering
\includegraphics[width=.5\columnwidth]{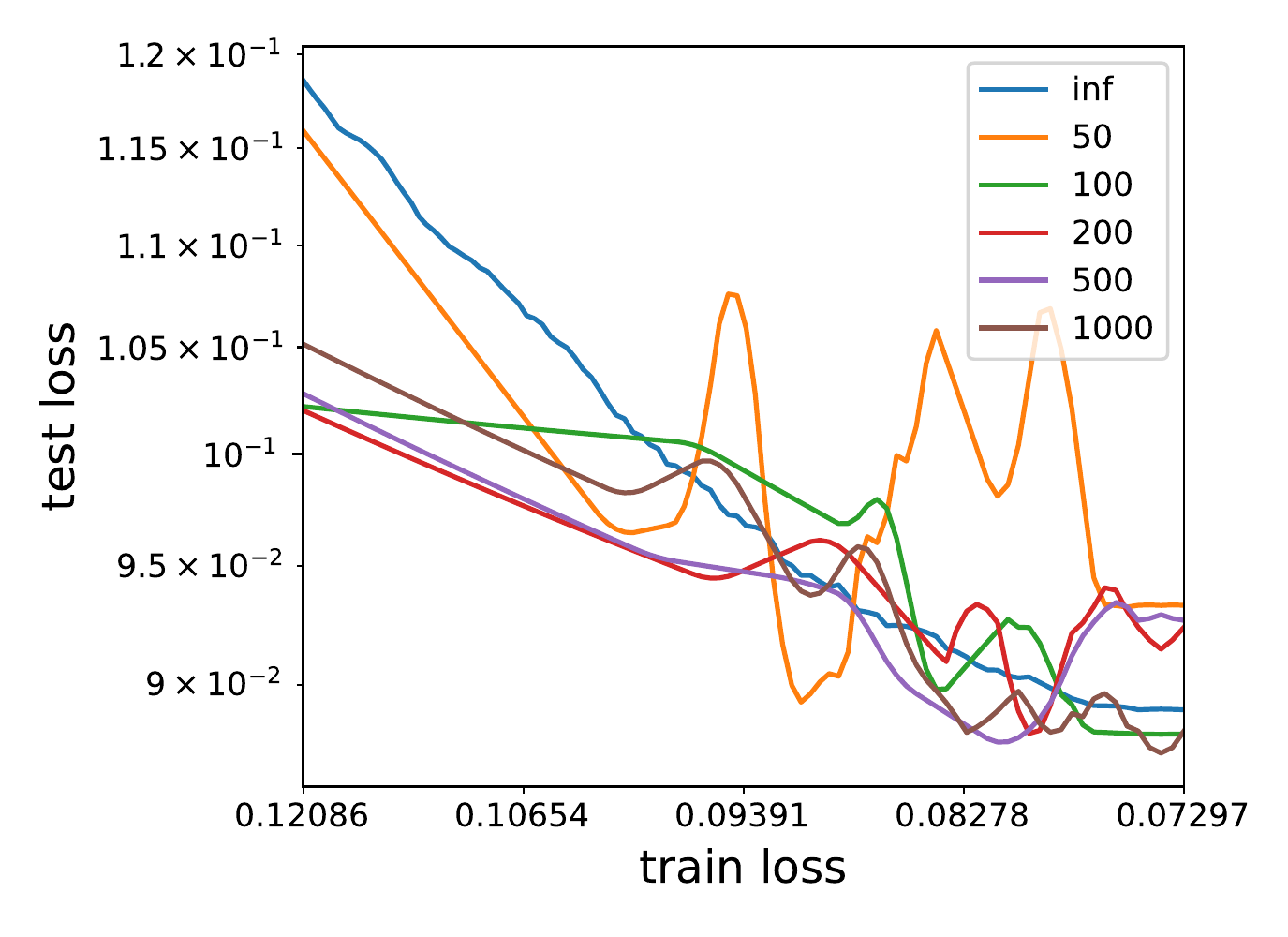}\\
  \caption{{CIFAR-10 dataset: training loss vs test loss plot}}%
  \label{fig:cifar10_lr1k_tvt}
  %\vspace{1cm}
\end{figure}

% To prove the claim we first establish the convergence of the dynamics

% of $g_n,\F_n$. For $g_n$, the empirical dynamical equation is given by the following ODE:
% \[
% \dot{g}_{n,t}(\tilde{\xi}) = \sum_{i=1}^N\Theta_{n,t}(\tilde{\xi},\xi_i)J_{n,t,i}^\top\chi_{n,t,i}
% \]
% Note that we may naively apply standard infinite width results for MLPs since $J_{n,t,i}$ depends on the weights $u,v$ only through the outputs $g$. Hence, we can construct an effective loss $\hat{\mathcal{L}}(g) = \mathcal{L}\big(f(g)\big)$, with a derivative $\nabla_g  \hat{\mathcal{L}}(g) = J(g)^\top \chi$. This concludes the proof of the updates to $g$ in the limit. The empirical dynamical equation for $\F$ is given in \cref{eqn:random_kernel}:
% \[
% \dot{\F}_{n,t}(\tilde{\xi}) = -\sum_{i=1}^N\mathcal{K}_{n,t}\big(\tilde{g}_{n,t},g_{n,t,i}\big)\chi_{n,t,i} 
% -\sum_{i=1}^N\tilde{J}_{n,t} \Theta_{n,t}(\tilde{\xi},\xi_i)J^{\top}_{n,t}(g_{n,t,i}) \chi_{n,t,i}
% \]
% In the infinite width limit both functions $\mathcal{K}_{n,t},\Theta_{n,t}$ converge almost surely to their respective limits $\mathcal{K},\Theta$. Moreover, both $\mathcal{K}(\tilde{g},g),$

% Since the Jacobian is defined as the derivative of the output $f(g)$ with respect to the input $g$, 

% For any input $g$, the Jacobian is defined as $J(g) = \frac{\partial f(g)}{\partial g}$. From the existence of the NN-GP in the limit of $n \to \infty$, we have that for any fixed input $g$, $\lim_{n \to \infty} f(g) = \mathring{f}(g)$ where $\mathring{f}$ is a gaussian process defined in \cref{eqn:nngp}. 

\section{Experimental verification of the theory}\label{verify}
Here we present results of a number of numerical simulations, serving as sanity checks on the theoretical results at finite but large $n$.  In order to numerically verify the results stated in \cref{lemma:GP}, i.e., Jacobian's GP behavior at initialization, we numerically estimate the Jacobian covariance and its cross- covariance with $f$ for a two-layer \text{ReLU} network with a large number of hidden units $n$ by averaging over initializations (100K for $J-J$, 30K for $J-f$ and $f-f$ each). Deviation of the empirical Jacobian covariances from their theoretical counterparts, given by \cref{JJ:cov}, \cref{fJ:cov} and \cref{sig1:sig2:def}, is subsequently measured for randomly generated input pairs as $\lVert cov_{empirical}-cov_{theory}\rVert_{F}/\lVert cov_{theory} \rVert_{F}$. Each plot in Figure \cref{fig:jj_cov_devs} shows deviations of a given covariance structure for up to 50 randomly generated input data pairs of dimension two and for different values of (large) $n$.\footnote{The main objective here is to only verify deviations/fractional errors are small for large $n$. In particular, we do not attempt to numerically establish convergence as a function of $n$.}

Next, we shift our attention to showing the evolution equations in \cref{{thm:main}} hold numerically for sample parameter values. The following direct strategy is used: We train four-layer \text{ReLU} networks with finite bottlenecks and large number of hidden units $n$, with architecture described in \cref{sec:bottle}, using full batch gradient descent (\text{GD}) and a small learning rate $lr=10^{-3}$ (to approximate a gradient flow) on a regression task with an \text{L2}-loss. Training data is synthetically generated. Instances are taken to be standard normal-distributed, while labels are assigned by a fixed random projection of the training data. In order to numerically track the network in the functional space during training, we checkpoint values of the network function $f$, $\chi$ (loss response) and the bottleneck embedding $g$, all evaluated on the training instances, as well as the Jacobian $J$ (for a fixed sample bottleneck embedding) at every training step during \text{GD}. The idea is to numerically show these snapshots satisfy the concentrated evolution equations \cref{eqn:eq1}, \cref{eqn:eq2} and \cref{eqn:eq3} at their corresponding time steps, if the \text{NTK} limiting \text{ReLU} kernels $\mathcal{K}$ and $\Theta$ are used and time derivatives in the evolution equations are discretized by forward finite differences (with time step size set to be the learning rate). Equivalently, we plot the corresponding residual error per step. Figure (\ref{fig:eom:res:errors}) summarizes the simulation results for a choice of parameters.  Each point in the Figure (\ref{fig:eom:res:errors}) plots is a residual error measured at a time step $t$ as follows $\lVert (\psi[t+1]-\psi[t])/lr-EQ_{\psi}^{rhs}[t]\rVert/\lVert EQ_{\psi}^{rhs}[t]\rVert$, where $\psi\in\{f, g, J\}$.\footnote{Besides finite $n$, continuous time discretization error (finite learning rate) contributes to the error.} Color-coding corresponds to different training samples on which the functional evolution equations were evaluated.

To verify the main result in \cref{lemma:linear}, we train two linear networks with different architectures of the type described in \cref{sec:implicit} for sample parameter values. One of the networks is taken to have $f_n$ and $g_n$ maps as linear \text{MLPs} of depth two. The second network has the simple form $\F(\xi) = w_{\text{eff}}\theta_\text{eff} \xi$ (as if the depth-two components of the first network are collapsed into one {\it effective} each). Similar to the previous case, we record the two networks dynamics in functional space during \text{GD}. We then numerically show the evolution equations \cref{eqn:eq1}, \cref{eqn:eq2} and \cref{eqn:eq3} are satisfied by the empirical trajectories in the functional space, when the appropriate NTK kernels for each {\it linear} network is used.  \cref{fig:eom:res:errors:4-layer:lin} and \cref{fig:eom:res:errors:2-layer:lin} show the evolution equations' residual error for the four and two layer linear networks described above, where kernels $\Theta(\xi,\tilde{\xi}) = L_g\frac{\xi^\top \tilde{\xi}}{d_0}, \mathcal{K}(x,\tilde{x}) = L_f\frac{x^\top \tilde{x}}{d}$ for $L_f=L_g=2$ and $\Theta(\xi,\tilde{\xi}) = \frac{\xi^\top \tilde{\xi}}{d_0}, \mathcal{K}(x,\tilde{x}) = \frac{x^\top \tilde{x}}{d}$ were used in the evolution equations respectively. The learning rate was set to $10^{-4}$. 

\begin{figure}[ht]
\centering
\begin{tabular}{cc}
\includegraphics[width=.3\columnwidth]{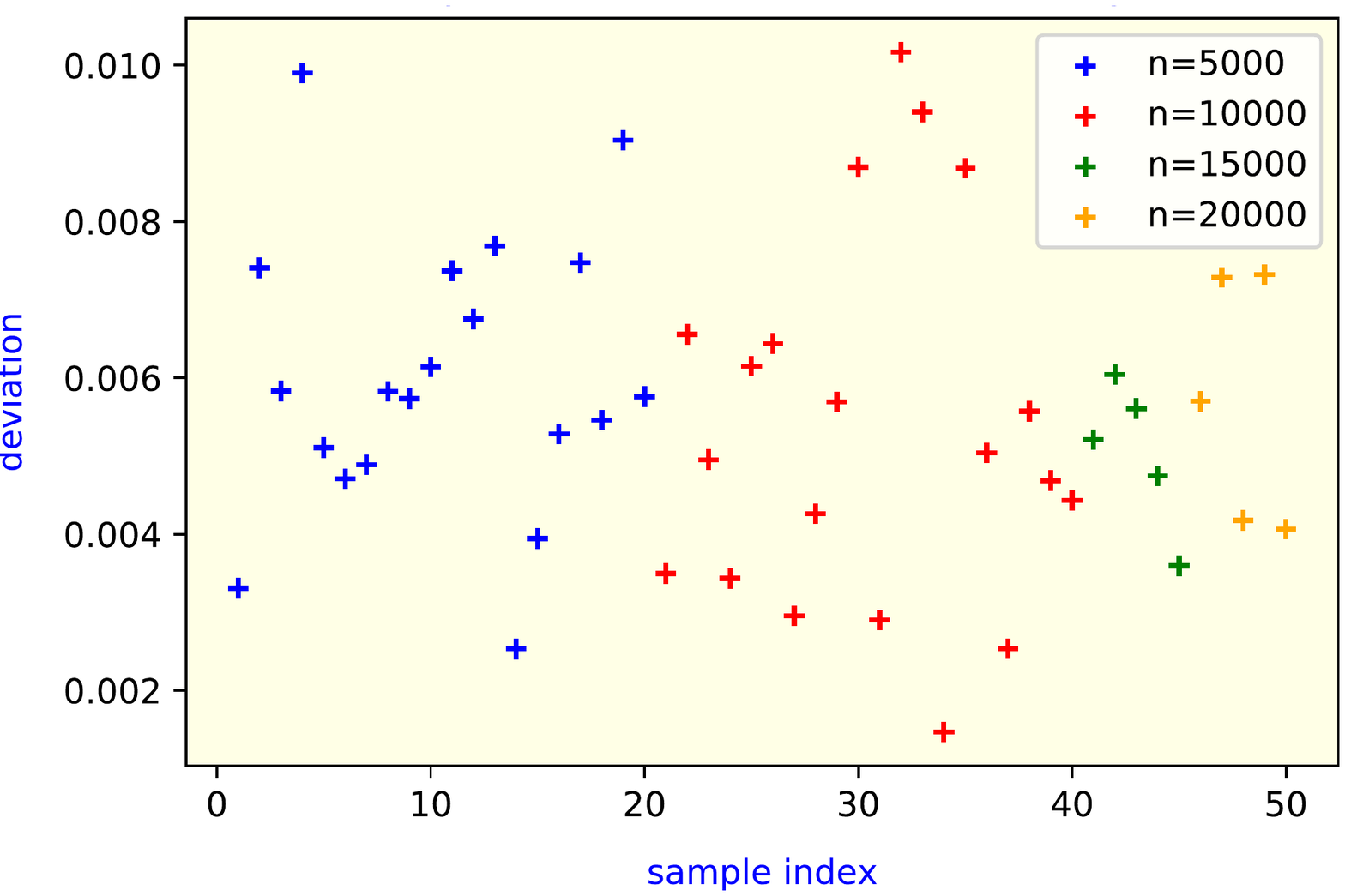}&
\includegraphics[width=.3\columnwidth]{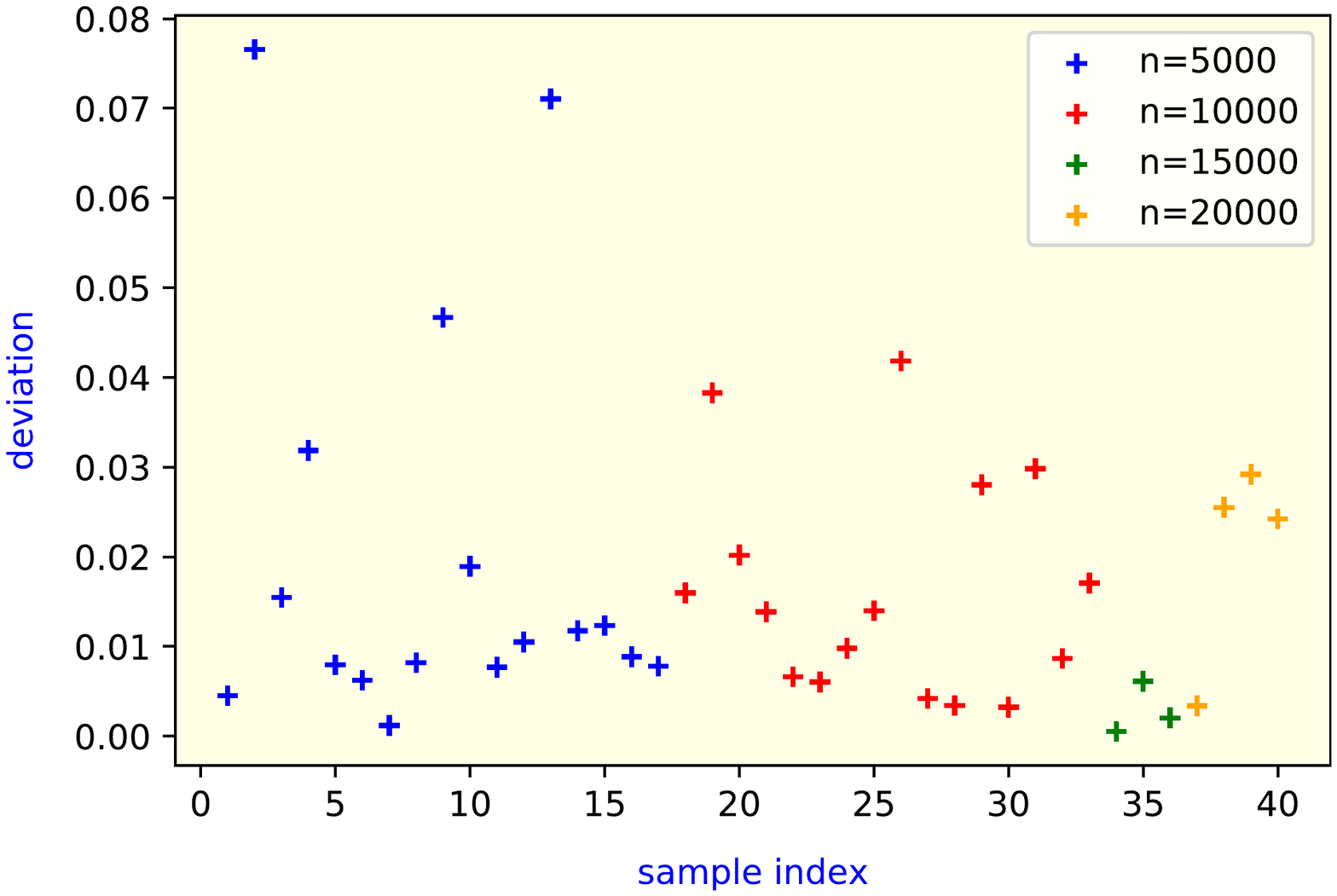}
\includegraphics[width=.3\columnwidth]{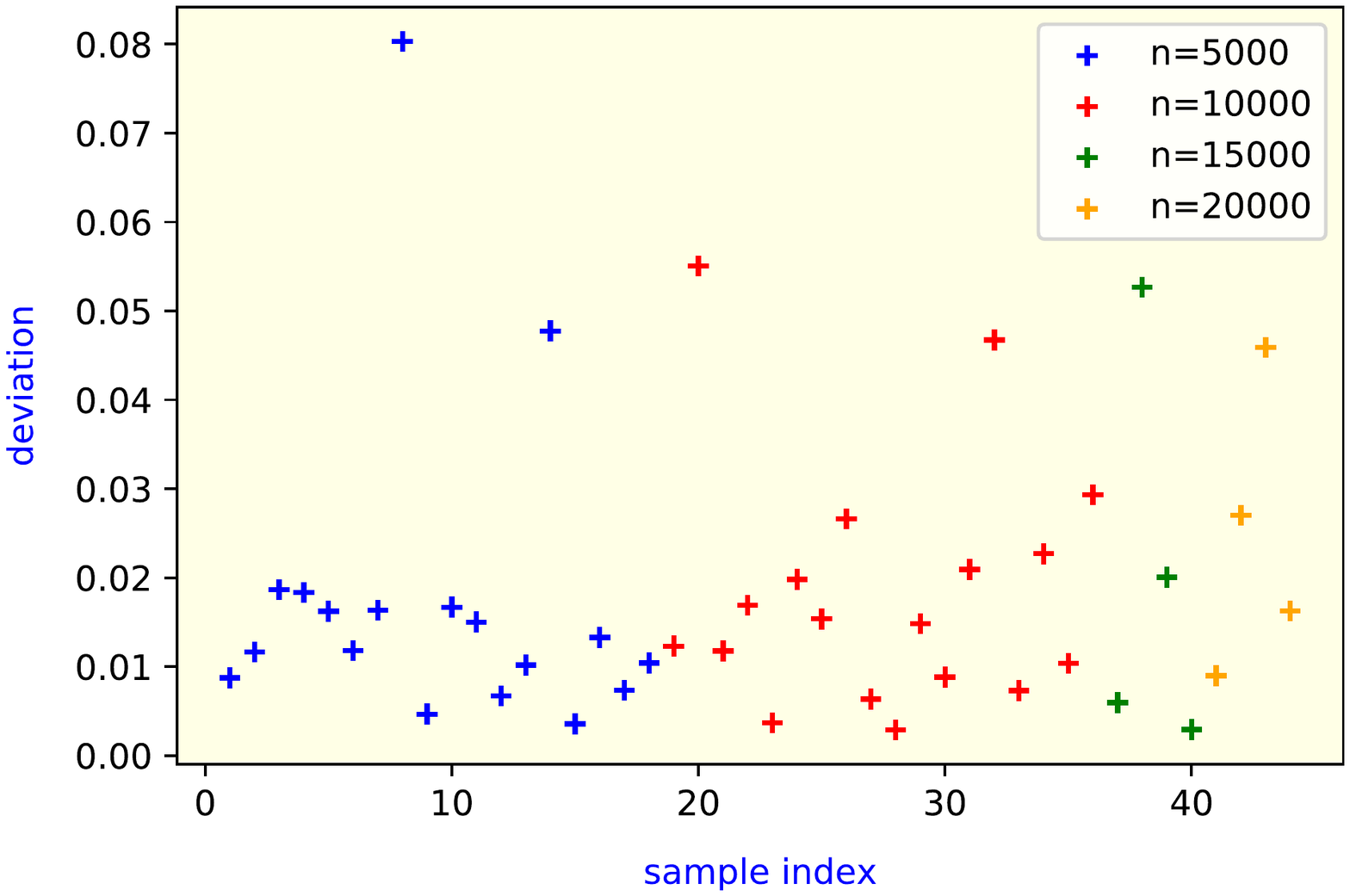}\\
(a) & (b)\\
\end{tabular}
  \caption{{Deviation of the experimental covariances from their theoretical predictions in (\ref{JJ:cov}) for pairs of random input samples for (a) $cov_{JJ}$ (b) $cov_{ff}$ and (c) $cov_{Jf}$. See \cref{verify} for details.}}%
  \label{fig:jj_cov_devs}
  %\vspace{-.2cm}
\end{figure}

\begin{figure}[ht]
\centering
\begin{tabular}{cc}
\includegraphics[width=.3\columnwidth]{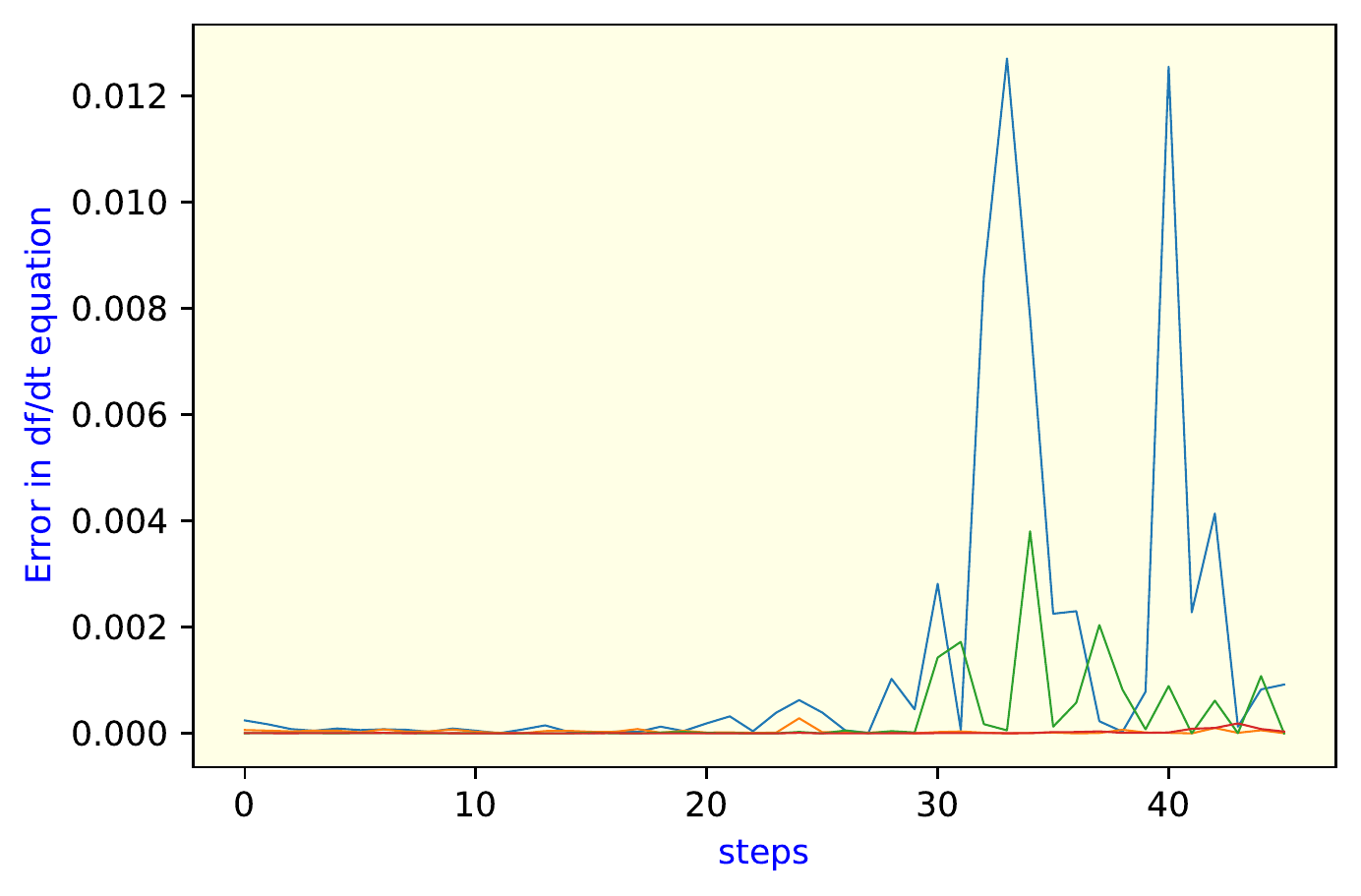}&
\includegraphics[width=.3\columnwidth]{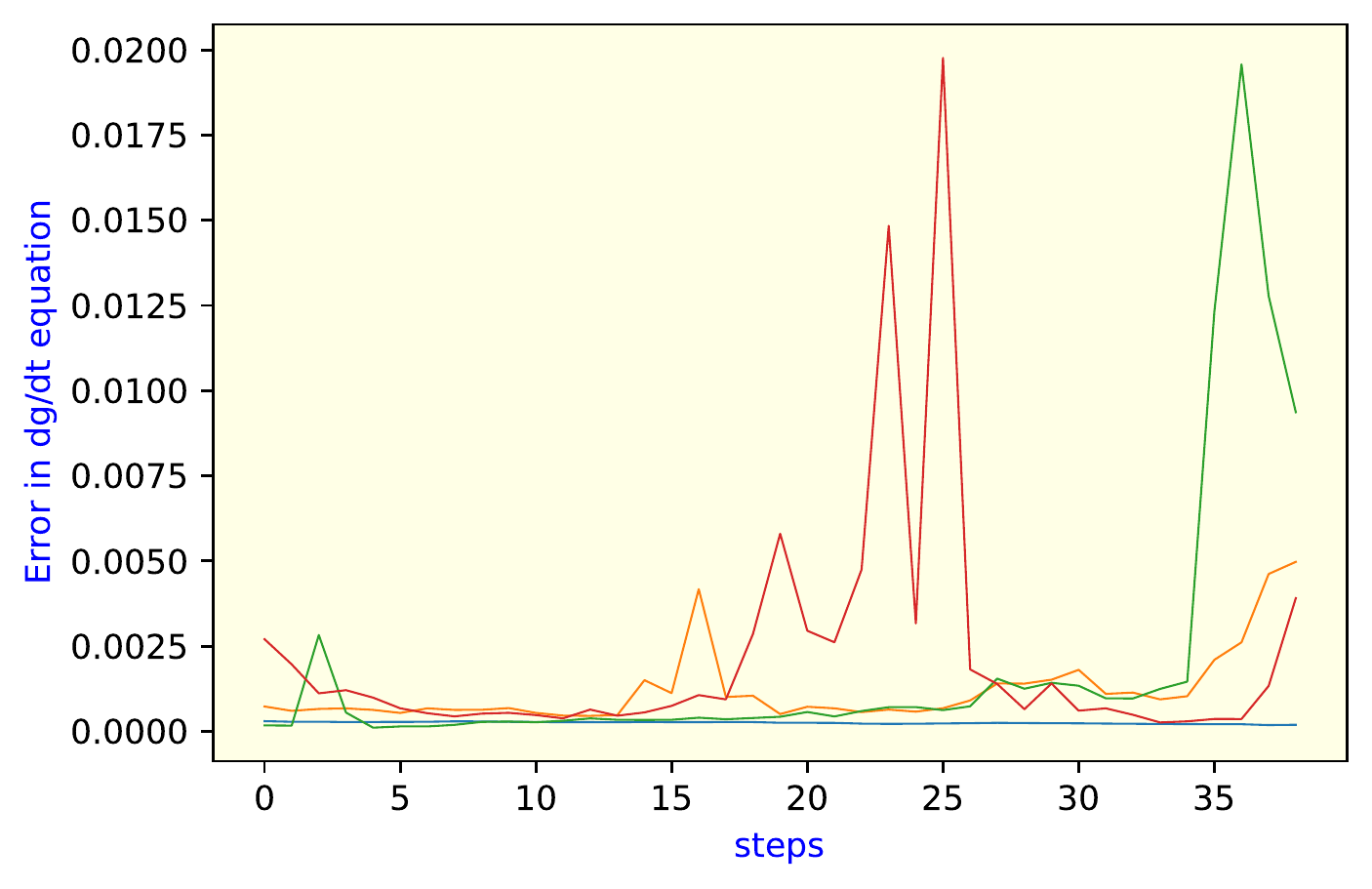}
\includegraphics[width=.3\columnwidth]{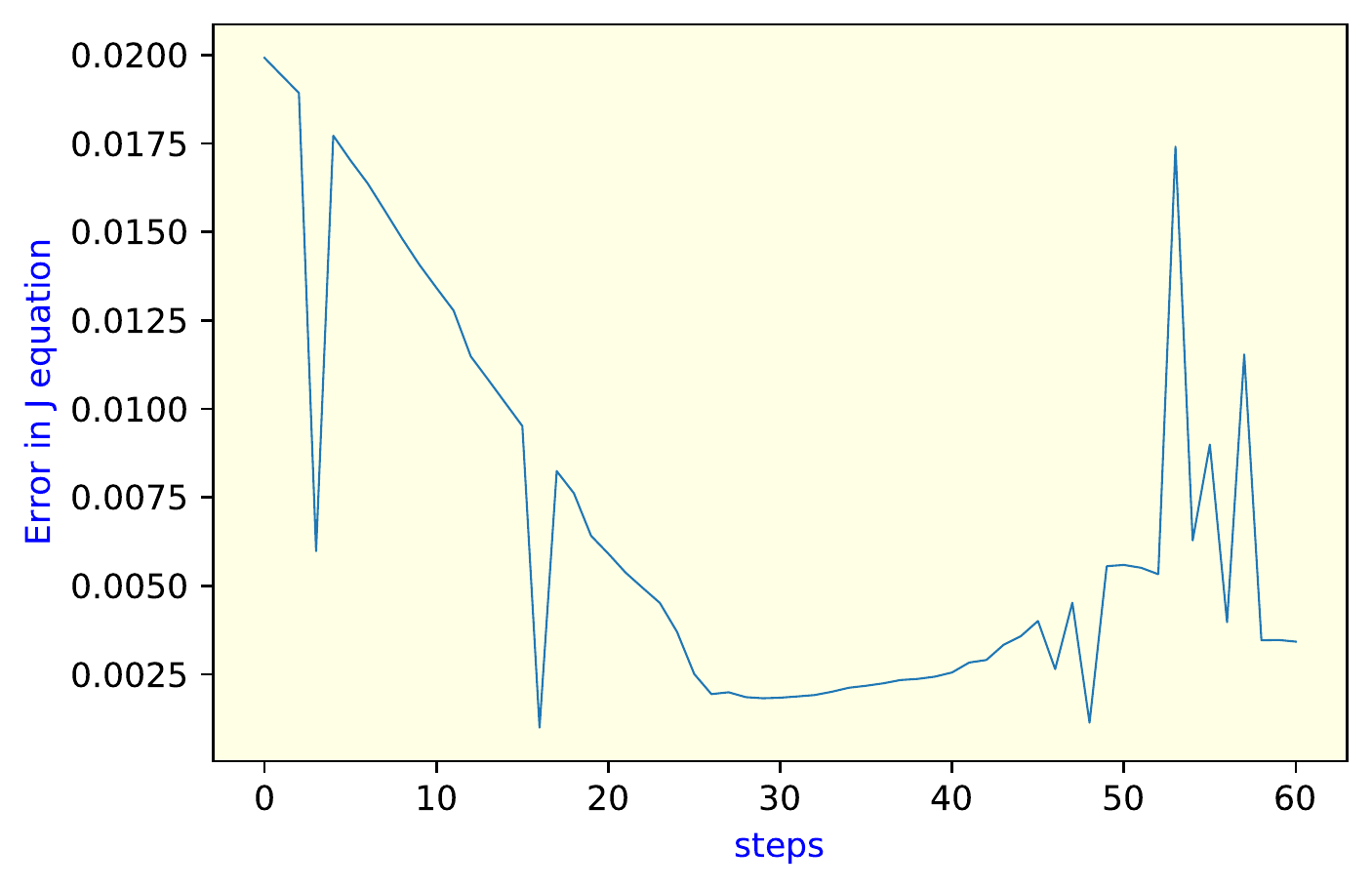}\\
(a) & (b)\\
\end{tabular}
  \caption{{Residual error for $f$, $g$ and $J$ evolution equations for $n=50K$ and a bottleneck size of three. See \cref{verify} for details.}}%
  \label{fig:eom:res:errors}
  %\vspace{-.2cm}
\end{figure}

\begin{figure}[ht]
\centering
\begin{tabular}{cc}
\includegraphics[width=.3\columnwidth]{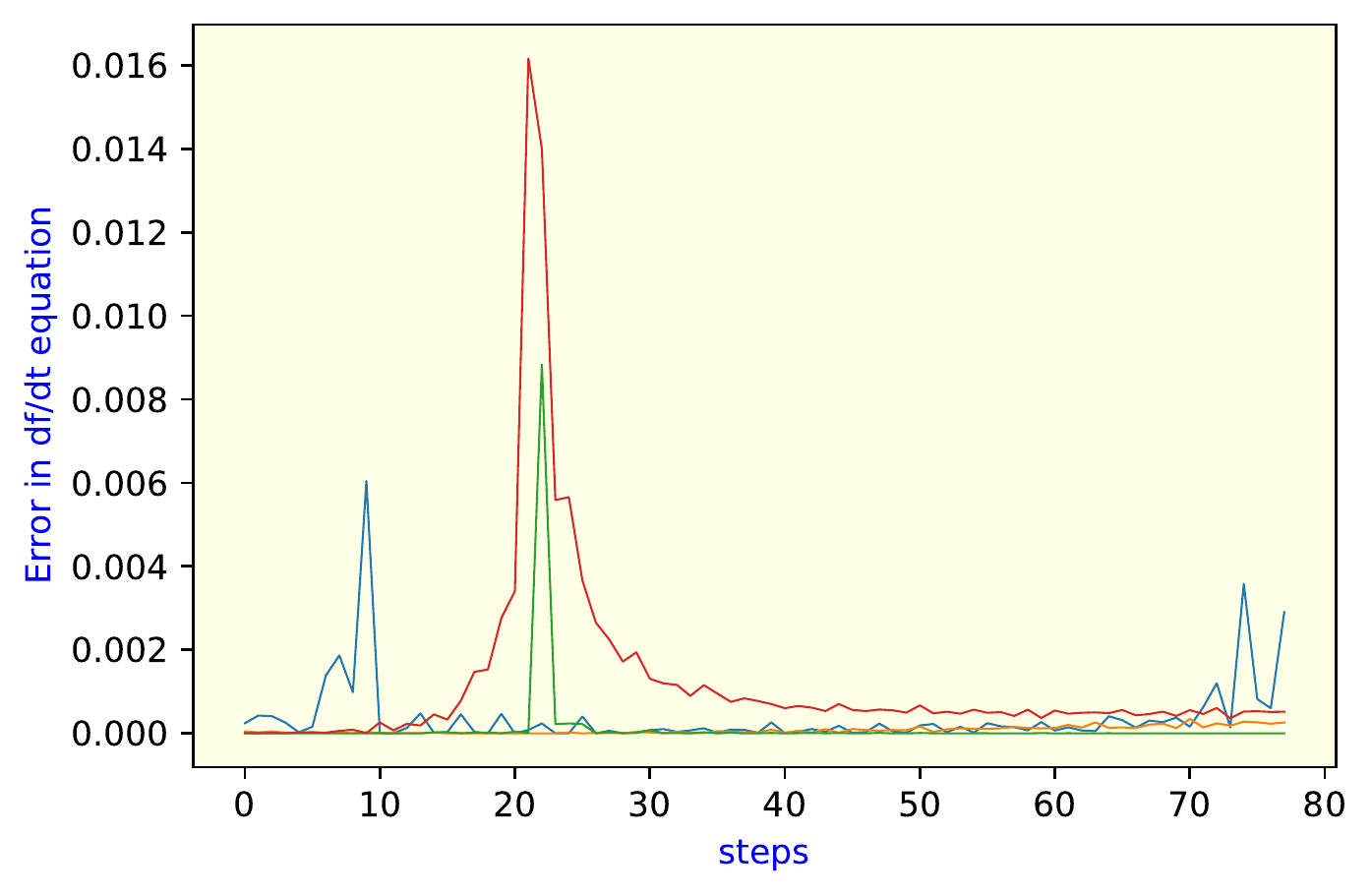}&
\includegraphics[width=.3\columnwidth]{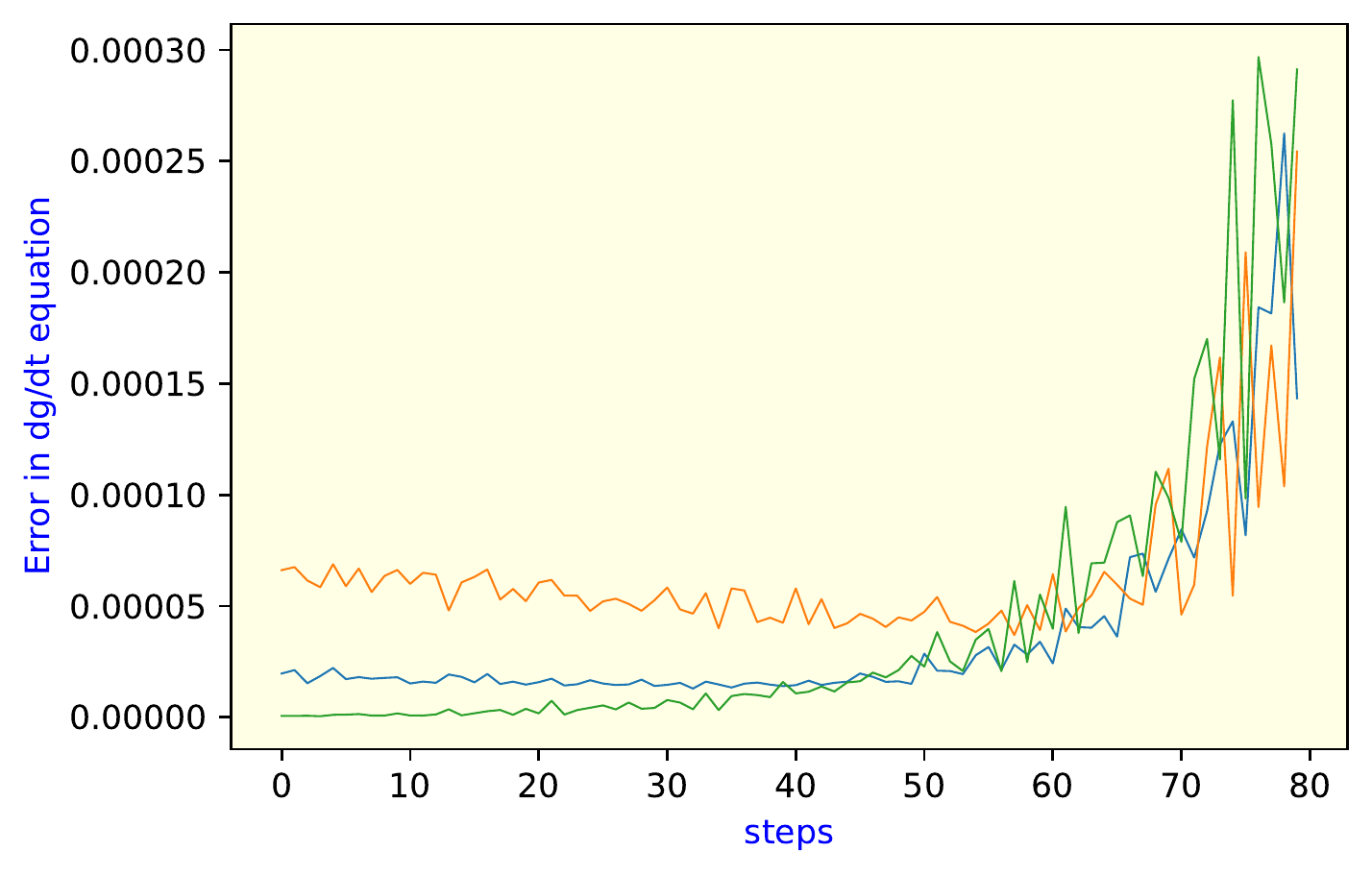}
\includegraphics[width=.3\columnwidth]{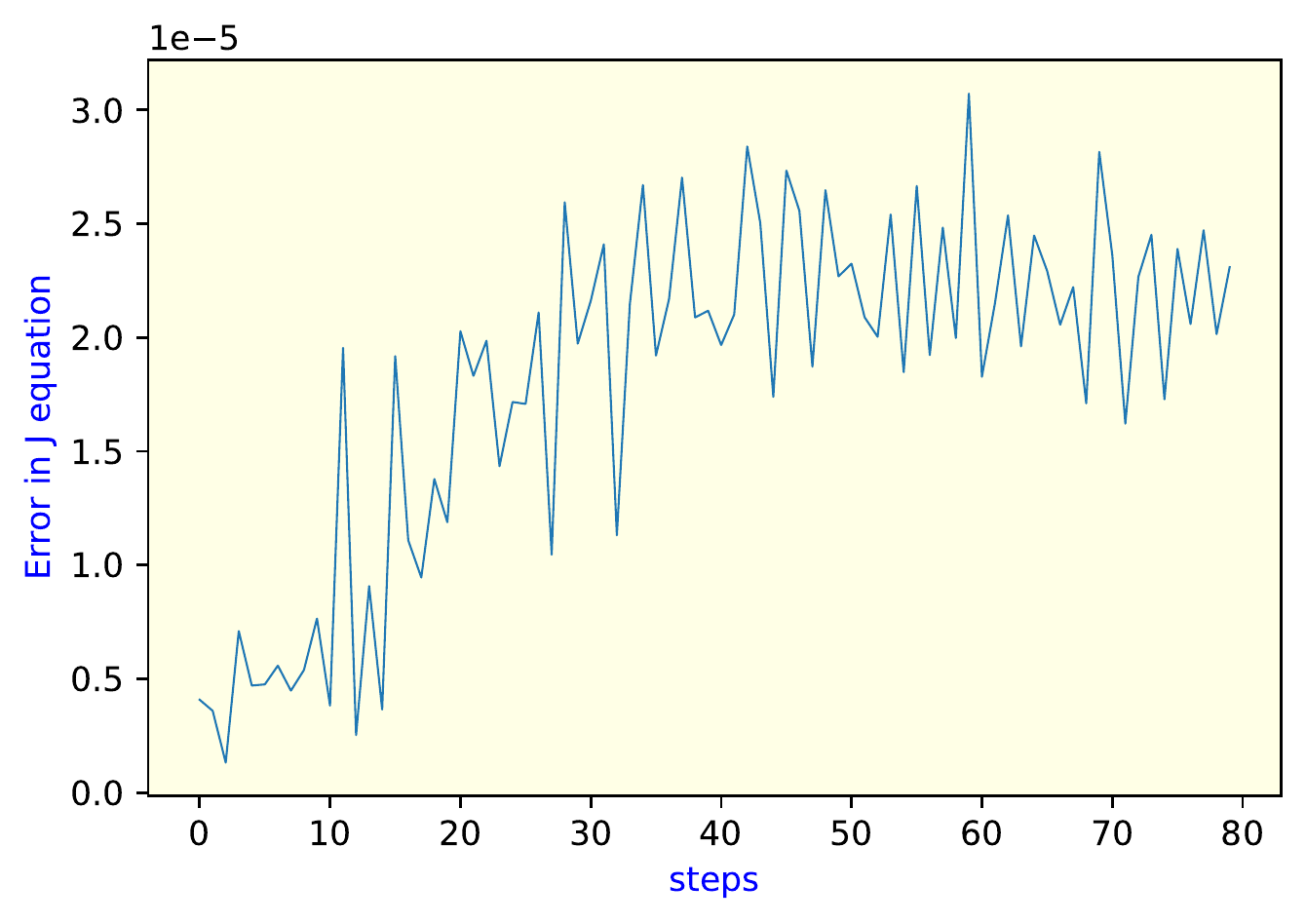}\\
(a) & (b)\\
\end{tabular}
  \caption{{Residual error for $f$, $g$ and $J$ evolution equations for a four-layer linear network with $n=50K$ and a bottleneck size of three. See \cref{verify} for details..}}%
  \label{fig:eom:res:errors:4-layer:lin}
  %\vspace{-.2cm}
\end{figure}

\begin{figure}[ht]
\centering
\begin{tabular}{cc}
\includegraphics[width=.3\columnwidth]{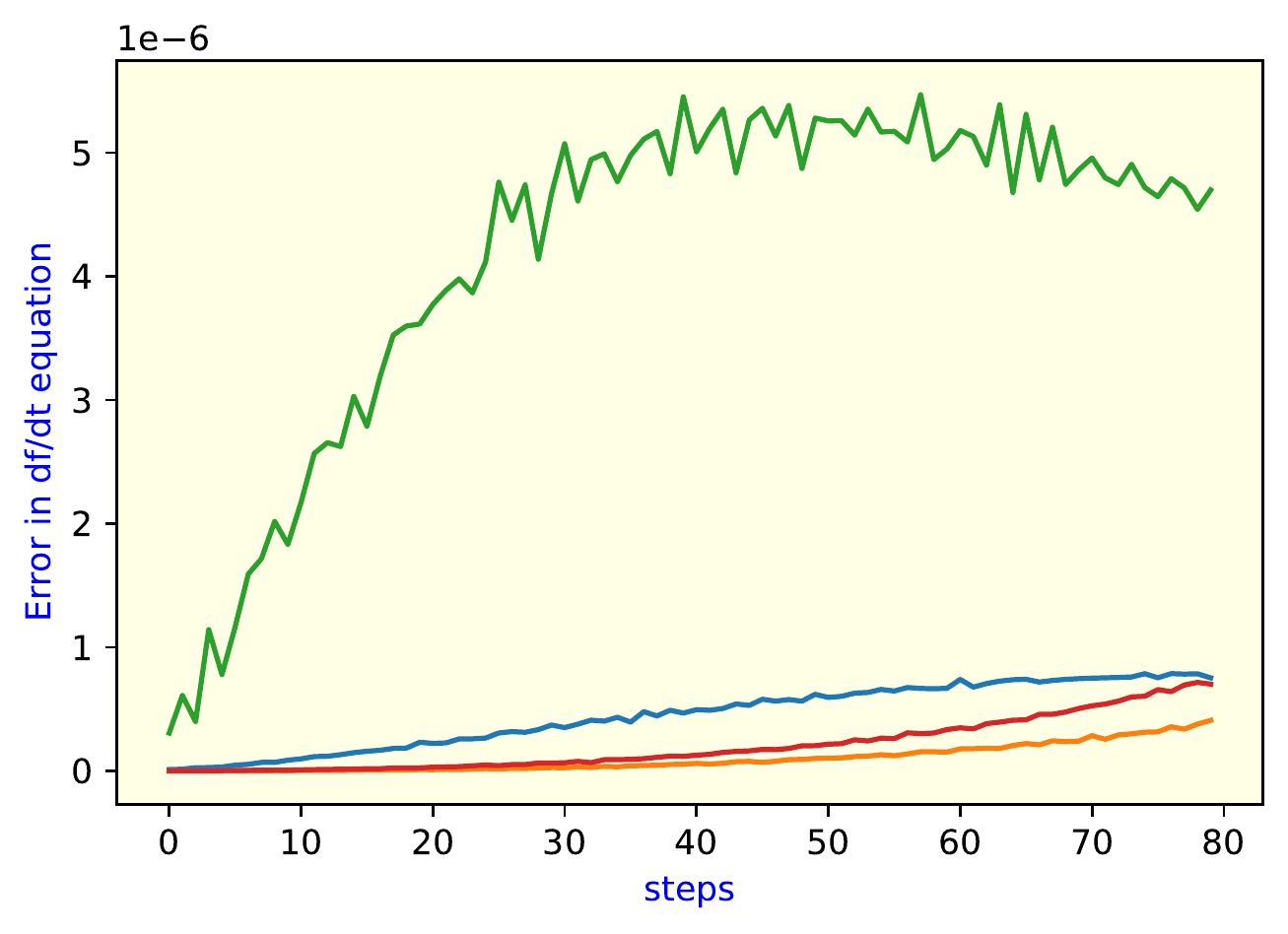}&
\includegraphics[width=.3\columnwidth]{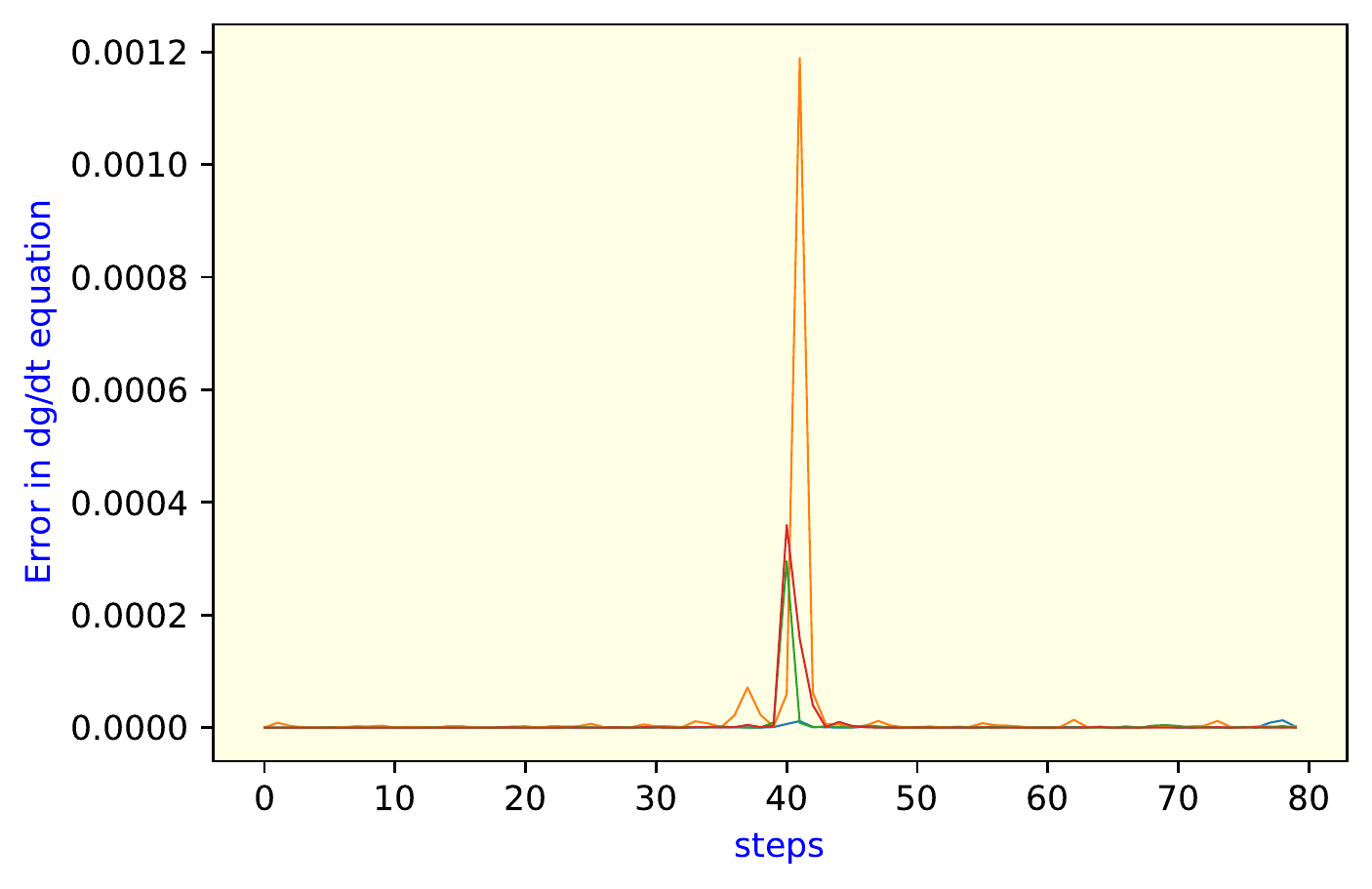}
\includegraphics[width=.3\columnwidth]{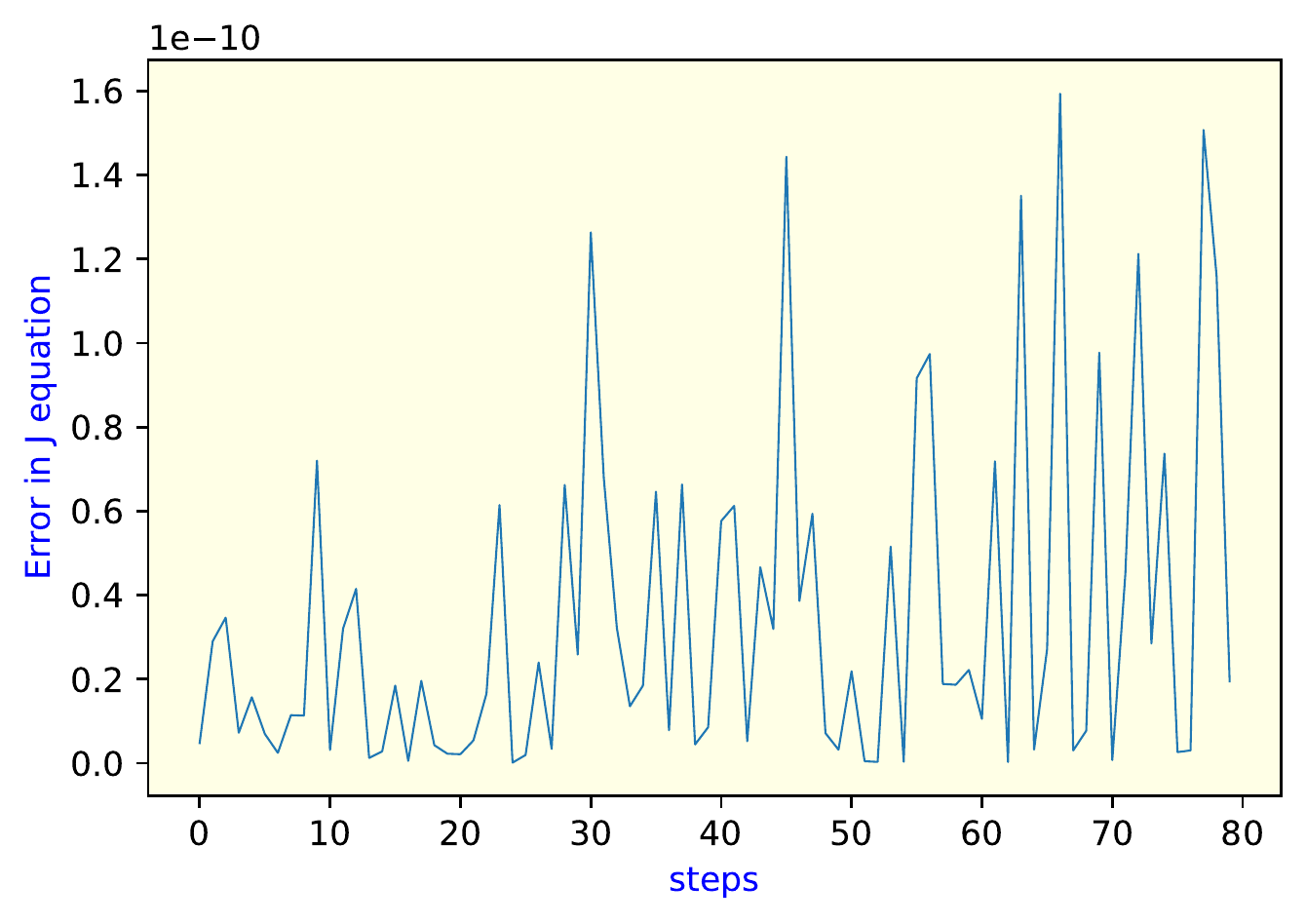}\\
(a) & (b)\\
\end{tabular}
  \caption{{Residual error for $f$, $g$ and $J$ evolution equations for a two-layer effective linear network described in \cref{verify}. Bottleneck size of three. Refer to \cref{verify} for details.}}%
  \label{fig:eom:res:errors:2-layer:lin}
  %\vspace{-.2cm}
\end{figure}

\newpage
\section{Finite approximation of infinite width networks}\label{finiteapprox}

We set up four-layer neural networks with a large number of dimensions $n$ (100K) to approximate infinite width layers, as well as a significantly smaller bottleneck layer between the wide layers. We used the standard batched stochastic gradient descent (SGD) with an L2-loss to train the model for a multioutput regression task on the MNIST dataset. The batchsize was reduced to 5 to allow for these wider networks to fit within GPU memory. We trained variants of the network with different bottleneck sizes and analyzed the change in losses during training (Figure~\ref{fig:finite_epoch}). The final loss and error rate (1-accuracy) at the end of training for 36 different bottleneck sizes: [1 to 9, 10 to 90, 100 to 900, 1000 to 9000] is shown in Figure~\ref{fig:finite_d} . For better clarity, Figure~\ref{fig:finite_epoch} shows the following selection of bottleneck sizes: [1, 2, 4, 8, 10, 20, 40, 80, 100, 200, 400, 800, 1000].

From the results, the best performing models have a relatively small bottleneck size. We observe that performance improves when the bottleneck size goes from wide to narrow, before degrading again when the width is extremely small as the representation power degrades. At the sweet spot, the model not only achieves lower training and test losses, it also trains faster. Interestingly, the bottleneck size with the lowest test loss did not result in the lowest test error (or highest accuracy). These observations are consistent with those in the infinite width bottleneck experiments (section~\ref{expt:data:mnist}).

% The best performing models have a relatively small bottleneck size, when measuring their loss or error rate.
%% SOME NOTES
% lower test/train loss as bottleneck as width goes from wide -> narrow
%finite width bottleneck models achieve lower test loss for a given training loss over infinite width bottlenecks
%narrower bottlenecks attain lower test loss compared to their wider finite width counterparts for all widths considered for this dataset
% lowest test loss doesn't correspond to best test accuracy
% Regarding the point that lowest test loss does not necessarily equate to highest test accuracy, we have two papers on that! https://arxiv.org/abs/1905.05895 https://arxiv.org/abs/2104.10631

%\begin{figure}[ht]
%\centering
%\includegraphics[trim=300 264 300 264,clip,width=0.5\textwidth]{workshop_figures/finite/finite_architecture.pdf}
%  \caption{{Architecture for approximating the infinitely wide network with very wide finite layers with large $n$ (100K).}}%
%  \label{fig:finite_architecture}
%\end{figure}

\begin{figure}[ht]
\centering
  \subfloat[Training loss]{\label{fig:finite_train_err}\includegraphics[trim=30 25 455 440,clip,scale=0.45]{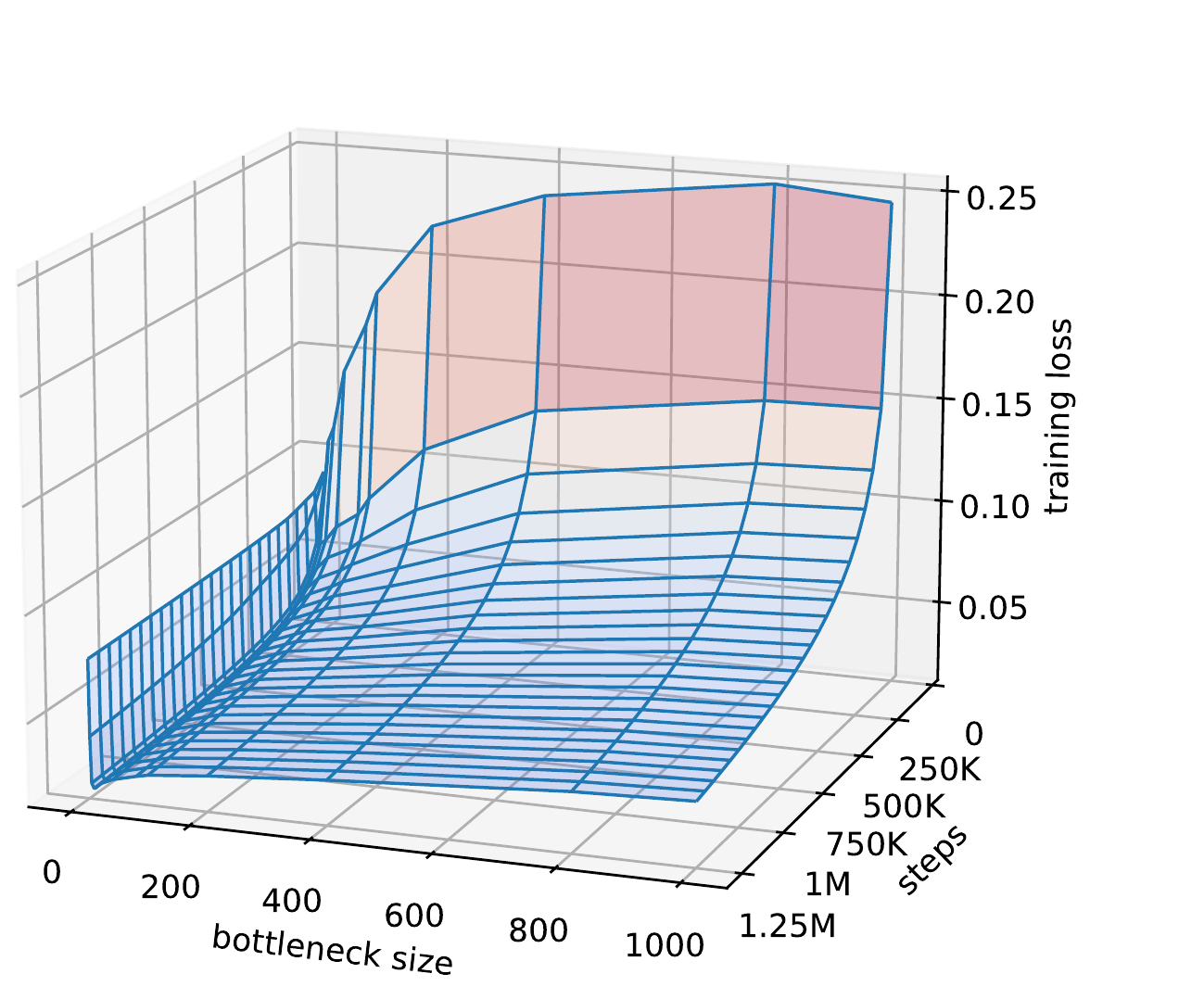}}
  \subfloat[Test loss]{\label{fig:finite_test_err}\includegraphics[trim=520 25 0 440,clip,scale=0.45]{workshop_figures/finite/finite_epoch_3d_type3.pdf}}\\
  %\subfloat[Training error rate]{\label{fig:finite_train_err}\includegraphics[trim=30 440 455 65,clip,scale=0.45]{workshop_figures/finite/finite_epoch_3d.pdf}}
  %\subfloat[Test error rate]{\label{fig:finite_test_err}\includegraphics[trim=520 440 0 65,clip,scale=0.45]{workshop_figures/finite/finite_epoch_3d.pdf}}
  \caption{Progression of training and test losses over training steps for $n=100K$ and bottleneck sizes from 1 to 1000.}
  \label{fig:finite_epoch}
\end{figure}

\begin{figure}[ht]
\centering
\subfloat[Loss]{\label{fig:finite_loss}\includegraphics[trim=0 0 425 0,clip,scale=0.45]{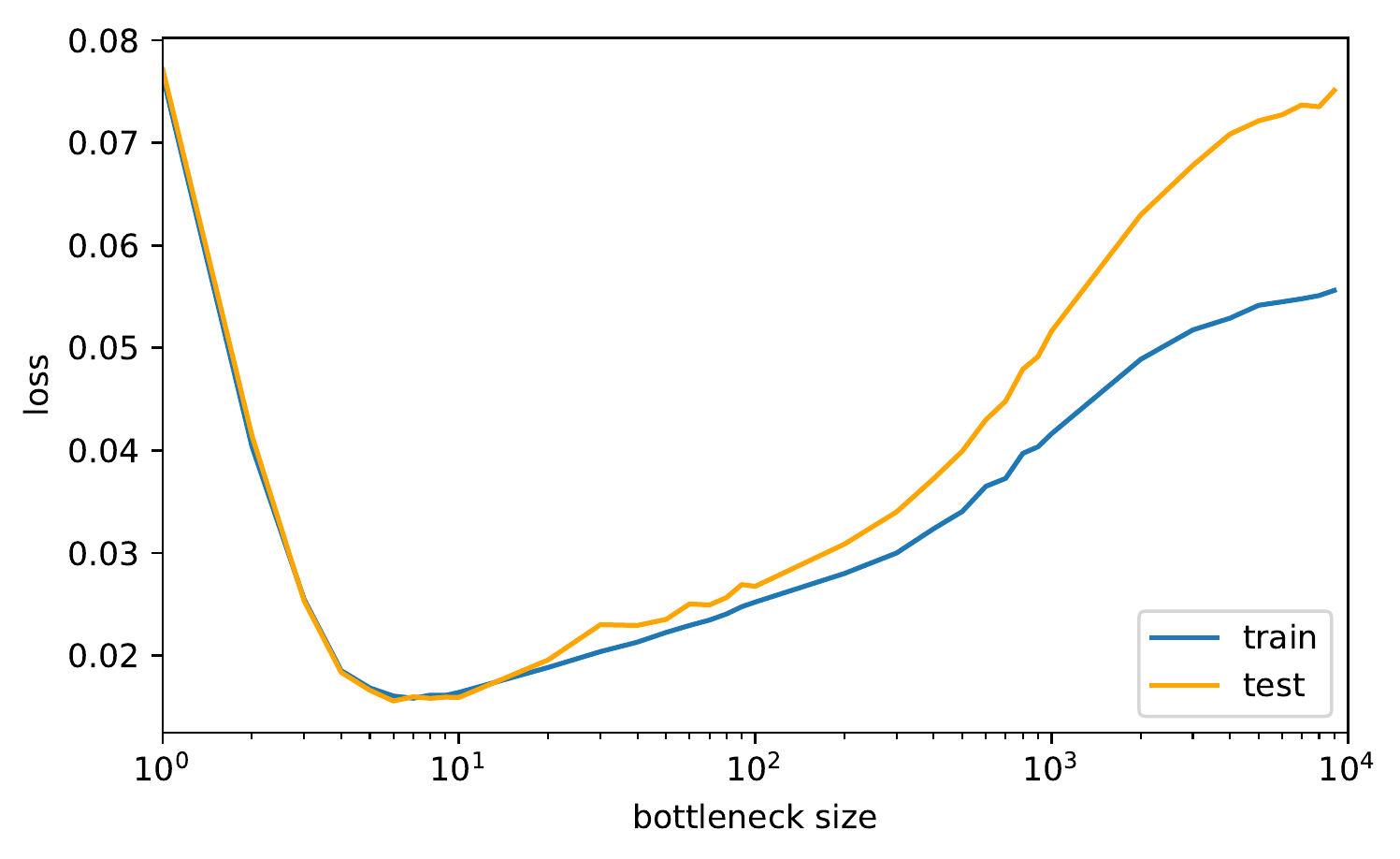}}
  \subfloat[Error rate (\%)]{\label{fig:finite_error}\includegraphics[trim=425 0 0 0,clip,scale=0.45]{workshop_figures/finite/finite_bottleneck_size_type3.pdf}}
  \caption{Loss and error rate w.r.t. bottleneck size for $n=100K$ and bottleneck sizes from 1 to 9000.}%
  \label{fig:finite_d}
\end{figure}

\end{document}